\documentclass[11pt]{article}

\usepackage{amsmath,amsfonts,amsthm,amssymb,color}
\usepackage[ruled,vlined]{algorithm2e}
\usepackage{fullpage}

\newtheorem{theorem}{Theorem}[section]
\newtheorem{corollary}[theorem]{Corollary}
\newtheorem{lemma}[theorem]{Lemma}
\newtheorem{proposition}[theorem]{Proposition}
\newtheorem{definition}[theorem]{Definition}
\newtheorem{problem}{Problem}

\usepackage{hyperref}
\usepackage{xspace}
\usepackage{ifthen}

\newcommand{\eps}{\ensuremath{\epsilon}\xspace}
\renewcommand{\tilde}{\widetilde}
\renewcommand{\hat}{\widehat}
\renewcommand{\bar}{\overline}

\newcommand{\R}{\mathbb{R}}

\DeclareMathOperator{\tr}{tr}
\newcommand{\norm}[1]{\left\|#1\right\|}
\newcommand{\normone}[1]{\norm{#1}_1}
\newcommand{\normtwo}[1]{\norm{#1}_2}
\newcommand{\norminf}[1]{\norm{#1}_\infty}
\newcommand{\fnorm}[1]{{\norm{#1}}_F}

\providecommand{\expect}[2]{\ensuremath{\ifthenelse{\equal{#1}{}}{\mathbb{E}}{\mathbb{E}_{#1}}\!\left[#2\right]}\xspace}
\providecommand{\prob}[2]{\ensuremath{\ifthenelse{\equal{#1}{}}{\Pr}{\Pr_{#1}}\!\left[#2\right]}\xspace}

\newcommand{\mini}[1]{\mbox{minimize} & {#1} &\\}
\newcommand{\maxi}[1]{\mbox{maximize} & {#1 } & \\}
\newcommand{\st}{\mbox{subject to} }
\newcommand{\con}[1]{&#1 & \\}

\newenvironment{lp}{\begin{equation}  \begin{array}{lll}}{\end{array}\end{equation}}
\newenvironment{lp*}{\begin{equation*}  \begin{array}{lll}}{\end{array}\end{equation*}}

\newcommand{\inner}[1]{\langle #1\rangle}

\usepackage{enumitem}

\newcommand{\NN}{\mathcal{N}}
\newcommand{\mus}{\mu^\star}
\newcommand{\abs}[1]{\left|{#1}\right|}

\DeclareMathOperator{\poly}{poly}
\DeclareMathOperator{\cov}{cov}
\newcommand{\OPT}{\mathrm{OPT}\xspace}

\newcommand{\mone}{\mathbf{1}}
\providecommand{\var}[1]{\ensuremath{{\rm Var}[#1]}\xspace}

\begin{document}

\title{Faster Algorithms for High-Dimensional \\ Robust Covariance Estimation}

\author{
Yu Cheng\thanks{Duke University. Email: \texttt{yucheng@cs.duke.edu}}\\
\and
Ilias Diakonikolas\thanks{University of Southern California. Email: \texttt{diakonik@usc.edu}}\\
\and
Rong Ge\thanks{Duke University. Email: \texttt{rongge@cs.duke.edu}}\\
\and
David P. Woodruff\thanks{Carnegie Mellon University. Email: \texttt{dwoodruf@cs.cmu.edu}}
}
\date{}


\maketitle

\makeatletter{}
\begin{abstract}
We study the problem of estimating the covariance matrix of a high-dimensional distribution when a small constant fraction of the samples can be arbitrarily corrupted.
Recent work gave the first polynomial time algorithms for this problem with near-optimal error guarantees for several natural structured distributions. 
Our main contribution is to develop faster algorithms for this problem whose running time nearly matches that of computing the empirical covariance.

Given $N = \tilde{\Omega}(d^2/\epsilon^2)$ samples from a $d$-dimensional Gaussian distribution, an $\epsilon$-fraction of which may be arbitrarily corrupted, our algorithm runs in time $\tilde{O}(d^{3.26})/\poly(\epsilon)$ and approximates the unknown covariance matrix to optimal error up to a logarithmic factor.
Previous robust algorithms with comparable error guarantees all have runtimes $\tilde{\Omega}(d^{2 \omega})$ when $\epsilon = \Omega(1)$, where $\omega$ is the exponent of matrix multiplication.
We also provide evidence that improving the running time of our algorithm may require new algorithmic techniques.
\end{abstract} 

\makeatletter{}
\section{Introduction}
\label{sec:intro}

Estimating the covariance matrix of a high-dimensional distribution (covariance estimation) 
is one of the most fundamental statistical tasks (see, e.g.,~\cite{bickel2008a, bickel2008b} and references therein). 
For a range of well-behaved distribution families, the empirical covariance matrix 
is known to converge to the true covariance matrix at an optimal statistical rate (with respect to various norms). 
For concreteness, suppose we are given $N$ independent samples from a centered 
Gaussian $\mathcal{N}(\mathbf{0}, \Sigma)$ on $\mathbb{R}^d$, with unknown 
covariance $\Sigma$, and we want to estimate $\Sigma$ with respect to the Frobenius norm. 
It is well-known (see, e.g., Section 4 of~\cite{cai2010} for an explicit reference) 
that the empirical covariance matrix has expected Frobenius error at most 
$O(d/\sqrt{N}) \cdot \|\Sigma\|_2$ from $\Sigma$, 
where $\| \cdot \|_2$ denotes the spectral norm; 
and this bound is the best possible, within a constant factor, among all $N$-sample estimators. 
Equivalently, after $N = \Omega(d^2/\eps^2)$ samples, the empirical covariance will have 
Frobenius error at most $\eps \cdot \|\Sigma\|_2$ with high constant probability.
This gives a computationally and statistically efficient covariance estimator for this fundamental setting.
(By Lemma~\ref{lem:fast-mult}, the empirical covariance can be computed in time $O(d^{3.26}/\eps^2)$,
which is the best known bound to date.)

In this paper, we study the outlier robust setting when a small constant fraction of our samples
can be arbitrarily corrupted. We work in the following model of corruptions (see, e.g.,~\cite{DKKLMS16})
that generalizes Huber's contamination model~(\cite{Huber64}):
 
\begin{definition}[$\eps$-Corruption] \label{def:adv}
Given $\eps > 0$ and a distribution family $\mathcal{D}$ on $\R^d$, 
the \emph{adversary} operates as follows: The algorithm specifies some number of samples $N$, 
and $N$ samples $X_1, X_2, \ldots, X_N$ are drawn from some (unknown) $D \in \mathcal{D}$.
The adversary is allowed to inspect the samples, removes $\eps N$ of them, 
and replaces them with arbitrary points. This set of $N$ points is then given to the algorithm.
We say that a set of samples is {\em $\eps$-corrupted}
if it is generated by the above process.
\end{definition}

More concretely, we study the following problem:
Given an $\eps$-corrupted set of $N$ samples from an unknown 
$\mathcal{N}(\mathbf{0}, \Sigma)$ on $\mathbb{R}^d$,
we want to compute an accurate estimate of $\Sigma$ in Frobenius norm 
(or in a stronger, affine invariant version that guarantees small total variation distance).
Note that even a single corrupted point can arbitrarily compromise the behavior 
of the empirical covariance matrix. Classical and more recent work in statistics 
has obtained minimax optimal robust covariance estimators. For example,~\cite{Rous85} 
proposed the minimum volume ellipsoid -- a natural generalization of the interquartile range -- 
and showed that it is provably robust in high-dimensions. More recently,~\cite{chen2018} proposes 
a (similarly robust) generalization of Tukey's median~(\cite{Tukey75}) for the covariance matrix.
We note that the information-theoretically optimal error for robustly estimating the covariance
of $\mathcal{N}(\mathbf{0}, \Sigma)$ in Frobenius norm is $O(\eps+d/\sqrt{N}) \cdot \|\Sigma\|_2$. 
That is, for $N = \Omega(d^2/\eps^2)$, one can estimate the covariance  
to accuracy $\Theta(\eps) \cdot \|\Sigma\|_2$, which is almost as well as in the non-contaminated setting.
Unfortunately, these estimators are hard to compute in general, i.e., their runtime 
scales exponentially with the dimension.

Recent work in TCS~(\cite{DKKLMS16, LaiRV16}) gave the first polynomial time robust estimators for a range of 
high-dimensional statistical tasks, including mean and covariance estimation. Since these initial 
papers~(\cite{DKKLMS16, LaiRV16}), a growing body of subsequent works have obtained polynomial-time 
robust learning algorithms for a variety of unsupervised and supervised high-dimensional models.
(See Section~\ref{ssec:prior} for more related work.) 

It should be noted that the aforementioned robust estimators have already been useful in exploratory data analysis. Specifically,~\cite{DKK+17} evaluated the robust covariance estimators of~\cite{DKKLMS16}~and~\cite{LaiRV16} 
to detect patterns in a well-known genetic dataset~(\cite{novembre2008genes}) in the presence
of corruptions. Perhaps surprisingly, it was found that the robust algorithms developed 
for $\mathcal{N}(\mathbf{0}, \Sigma)$ outperformed all previous approaches on this real dataset, essentially 
matching the setting where there are no corruptions at all.

Once a polynomial-time algorithm for a computational problem has been discovered,
the next step is to focus on designing asymptotically faster algorithms for the problem
-- with linear time as the ultimate goal. We note that the aforementioned robust estimators~(\cite{DKKLMS16, LaiRV16})
are significantly slower than their non-robust counterparts (e.g., computing the empirical mean/covariance),
hence may not be scalable when the dimension is very high. This raises the following natural question:
\vspace{-0.1cm}
\begin{quote}
{\em Can we design robust estimators that are as efficient as their non-robust analogues?} 
\end{quote}
\vspace{-0.1cm}
This direction was initiated in~\cite{ChengDR19} who gave a robust mean estimation algorithm
with runtime $\tilde{O}(N d)/\poly(\eps)$,\footnote{The $\tilde{O}(\cdot)$ notation hides logarithmic factors in its argument.} nearly matching the runtime of computing the empirical mean (when $\eps$ is constant). 

In this work, we continue this line of investigation. At a high-level, our main contribution is the first robust covariance 
estimator whose running time nearly matches that of computing the empirical covariance matrix.
Moreover, we provide evidence that the runtime of our algorithm may not be improvable
with current algorithmic techniques.
In more detail, on input an $\eps$-corrupted set of $N = \tilde{O}(d^2/\eps^2)$ samples from  
$\mathcal{N}(\mathbf{0}, \Sigma)$ on $\mathbb{R}^d$, our algorithm runs in time $\tilde{O}(d^{3.26})/\poly(\eps)$,
and outputs a covariance estimate with near-optimal error guarantee, matching the one in~\cite{DKKLMS16}
(see Theorem~\ref{thm:main}). Our algorithm uses the primal-dual framework of~\cite{ChengDR19} recently developed for robust mean estimation, 
with a number of crucial twists that are required 
for the more challenging task of covariance estimation (see Section~\ref{sec:overview}).

For the sake of direct comparison, we note that the filtering-based robust covariance
estimator of~\cite{DKKLMS16} has runtime $\Omega(N^2 d) = \Omega(d^5)/\poly(\eps)$. On the other hand, 
the recursive dimension-halving estimator of~\cite{LaiRV16} requires $\Omega (\log d)$ SVD computations
of a $d^2 \times d^2$ ``covariance'' matrix, hence has runtime $\Omega(d^{2 \omega})$, where $\omega$ 
is the exponent of matrix multiplication. (Plugging in the best-known value for $\omega$~(\cite{Gall14a}) gives a runtime
of $\Omega(d^{4.74})$.)

We note that the runtime of our algorithm, while being super-linear, 
essentially matches the best-known runtime to compute the empirical covariance matrix (see Section~\ref{sec:hardness-matrix}). 
Moreover, we provide evidence (Section~\ref{sec:hardness-weights}) that this runtime may be a bottleneck even for 
the weaker task of obtaining an implicit representation to the output (by reweighing the input samples).
It should be noted that all known computationally efficient robust estimators fit in this framework.

\subsection{Our Results} \label{ssec:results}

Our first algorithmic result states that we can robustly estimate the covariance matrix 
of a high-dimensional Gaussian within multiplicative, dimension-independent error,
with running time that almost matches that of computing the empirical covariance matrix.

\begin{theorem}[Robust Covariance Estimation (Multiplicative)]
\label{thm:main}
Let $D \sim \mathcal{N}(\mathbf{0}, \Sigma)$ be a zero-mean unknown covariance Gaussian on $\R^d$.
Let $\kappa$ denote the condition number of $\Sigma$.
Let $0 < \eps < \eps_0$ for some universal constant $\eps_0$.
Given as input an $\eps$-corrupted set of $N = \tilde \Omega(d^2 / \eps^2)$ samples drawn from $D$,
  there is an algorithm (Algorithm~\ref{alg:main}) that runs in time $\tilde O(d^{3.26} \log(\kappa))/\poly(\eps)$ and outputs $\hat \Sigma \in \R^{d \times d}$ such that 
with high probability it holds $\| \Sigma^{-1/2} \hat \Sigma \Sigma^{-1/2} - I \|_F \le O(\eps \log(1/\eps))$.
\end{theorem}

We also develop a related robust covariance estimation algorithm (Algorithm~\ref{alg:additive}) 
with additive error guarantee, whose running time does not depend on the condition number of $\Sigma$.

\begin{theorem}[Robust Covariance Estimation (Additive)]
\label{thm:main-additive}
For the same setting as in Theorem~\ref{thm:main},
there is an algorithm (Algorithm~\ref{alg:additive}) that
runs in time $\tilde O(d^{3.26})/\poly(\eps)$ and outputs $\hat \Sigma \in \R^{d \times d}$ such that 
with high probability it holds $\| \hat\Sigma - \Sigma \|_F \le O(\eps \log(1/\eps)) \normtwo{\Sigma}$.
\end{theorem}

We will prove Theorem~\ref{thm:main} in Section~\ref{sec:main} and Theorem~\ref{thm:main-additive} in Appendix~\ref{sec:main-additive}.

In Section~\ref{sec:hardness}, we provide evidence that the runtime of our algorithm may be difficult
to improve. Specifically, we show that the best-known runtime for computing (or even approximating)
the empirical covariance is $\Omega(d^{3.25})$. Moreover, even outputting a set of weights for the samples
(such that the weighted empirical covariance works) seems to require $\Omega(d^{3.25})$ time with current 
methods.

\makeatletter{}
\subsection{Our Approach and Techniques}
\label{sec:overview}

When $X \sim \NN(0, \Sigma)$, we have $\expect{}{X X^\top} = \Sigma$, so at a high level, we want to reduce the robust covariance estimation problem to the problem of robustly estimating the mean of the $d^2$-dimensional random variable $Z = X \otimes X$. However, there are two main difficulties in this reduction.

\paragraph{Faster Positive SDP Solvers for Tensor Input.}
The first difficulty is that the input of the mean estimation problem is now a set of $d^2$-dimensional vectors. 
Even just computing all of these vectors explicitly will take time $\Omega(Nd^2)$. Our algorithm needs to solve the robust mean estimation problem for 
$X\otimes X$ {\em without computing these vectors explicitly}. To achieve that, we 
adapt the approach in \cite{ChengDR19}. Given data points $Z_i = X_i\otimes X_i$, 
the algorithm in \cite{ChengDR19} starts with a guess $\nu \in \R^{d^2}$, 
and approximately solves the following two (dual) SDPs at every iteration:

\begin{lp}
\label{eqn:primal-sdp}
\mini{\lambda_{\max} \left(\sum_{i=1}^N w_i (Z_i - \nu) (Z_i - \nu)^\top\right)} 
\st \con{\sum_{i=1}^N w_i = 1, \quad \forall i, 0\le w_i \le \frac{1}{(1-\eps)N}} \end{lp}
\begin{lp}
\label{eqn:dual-sdp}
\maxi{\text{average of the smallest $(1-\eps)$-fraction of $\left((Z_i - \nu)^\top M (Z_i - \nu)\right)_{i=1}^N$}}
\st \con{M \succeq 0, \tr(M) \le 1} \end{lp}

\cite{ChengDR19} showed the following win-win phenomenon: 
If the primal SDP (\ref{eqn:primal-sdp}) has a good solution, then it gives weights $w \in \R^N$ such that the weighted average $\sum_{i=1}^N w_i Z_i$ is close to the true mean.
Otherwise, if (\ref{eqn:primal-sdp}) does not have a good solution, 
then the dual SDP (\ref{eqn:dual-sdp}) gives a direction of improvement that allows the algorithm to find $\nu'$ that is much closer to the true mean. 
In our setting of covariance estimation, writing out the matrix of $(Z_i)_{i=1}^N$ already takes $O(Nd^2)$ time.
 We circumvent this problem by ``opening up'' the fast positive SDP solver (\cite{PengTZ16}) they used. 
The slowest step of this SDP solver is to compute multiplicative approximations 
of the values $\exp(\Psi)\bullet Z_iZ_i^\top$ for all $i$, where $\Psi = \sum_{i=1}^N w_i Z_iZ_i^\top$.
We exploit the structure of $Z_i = X_i\otimes X_i$ and design faster algorithms for getting such approximations. 
More precisely, let $\bar{Z}$ be a $N\times d$ matrix whose $i$-th column is $\sqrt{w_i} Z_i$, so that 
$\Psi = \bar{Z}\bar{Z}^\top$. We express $\exp(\Psi)$ as a low-degree polynomial over $\bar{Z}$ and $\bar{Z}^\top$, 
then use fast matrix multiplication (Lemma~\ref{lem:fast-mult}) and the transposition principle (Lemma~\ref{lem:lr-mult}) 
to show it is possible to right-multiply a vector with $\bar{Z}$ and $\bar{Z}^\top$ efficiently.

\paragraph{Iterative Refinement.}
The second difficulty is that existing robust mean estimation algorithms 
rely on the assumption that the distribution of the good data either has 
known covariance or unknown bounded covariance. By reducing 
covariance estimation to the mean estimation of $X\otimes X$, we run into
the difficulty that
the covariance of such vectors would correspond to the 
fourth order moments of the original variables $X$. So, directly 
applying the mean estimation algorithms does not give our desired strong guarantees.

We solve this problem using iterative refinement steps. 
Similar iterative refinement steps were discussed in~\cite{Kane18}, which does not seem sufficient for our purposes, 
as it may require a linear number of iterations.
We adapt this approach to obtain faster algorithms.
More precisely,
  given an upper bound $\Sigma_t$ on the true covariance matrix $\Sigma$, we can compute a more accurate upper bound $\Sigma_{t+1} \succeq \Sigma$.
We use two different types of iterative refinement steps (see Lemmas~\ref{lem:first-loop}~and~\ref{lem:second-loop}). Both refinement steps first rotate the input samples $Y_i = \Sigma_t^{-1/2} X_i$ 
and compute the Kronecker products $Z_i = Y_i \otimes Y_i$.
When $X \sim \NN(0, \Sigma)$, the covariance matrix of $Y$ is $\Sigma_t^{-1/2} \Sigma \Sigma_t^{-1/2}$,
  and the mean of $Z$ is $\expect{}{Z} = \left(\Sigma_t^{-1/2} \Sigma \Sigma_t^{-1/2}\right)^\flat$.
Note that if we had $\expect{}{Z}$, we could recover $\Sigma$ from $\Sigma_t$ immediately.
We show that a good estimate of $\expect{}{Z}$ will help us improve $\Sigma_t$ and get a better upper bound $\Sigma_{t+1}$.

We establish two guarantees for the robust estimation of the mean of $Z$. 
In the first phase, we only use the fact that the covariance of $Z$ is bounded. 
Repeating this will give a rough estimation of the covariance of $X$. 
In the second phase, we use the fact that the current estimate $\Sigma_t$ 
is already very close to $\Sigma$, therefore $Y_i = \Sigma_t^{-1/2} X_i$ 
is very close to a standard Gaussian $N(0,I)$ for the good samples. 
In this case, we need to open up the algorithm in \cite{ChengDR19} 
and prove stronger robust estimation guarantees tailored for the 
specific distribution of $Z = Y\otimes Y$. Our main algorithm (Algorithm~\ref{alg:main}) 
combines these two refinement steps to match the best-known robustness guarantees for covariance estimation. 
\paragraph{Evidence of Hardness.}
It is natural to ask whether one can obtain faster running times than
those achieved here. There is a sample complexity lower bound of $N =
\Omega(d^2)$, so we assume we are given $X \in \mathbb{R}^{N \times
d}$ as our sample matrix (here we will focus on constant $\epsilon$).
We note that even in the non-robust setting, it is not known how to
output an approximate covariance matrix $\hat{\Sigma}$ for which
$\|\hat{\Sigma} - \frac{1}{N} X^T X\|_F = O(1)$ in time faster than
$(d, d^2, d)$ matrix multiplication time, which is the time to
multiply a $d \times d^2$ matrix by a $d^2 \times d$ matrix. Since the
non-robust case is a special case of our setting, one cannot improve
our running time without improving the running time of non-robust
covariance estimation.

A natural way of improving the running time of non-robust covariance
estimation is to try to approximate the product $X^T X$ using
oblivious sketching (see, e.g., \cite{Woodruff14} for a survey) which
works roughly as follows. One samples a random $S$ from a certain
family of random matrices, and computes $S \cdot X$ where $S$ has much
fewer than $N$ rows. For structured random families of matrices, like
fast JL matrices, $S \cdot X$ can be computed very quickly. Then one
instead computes $X^TS^TS X$ with the guarantee that $\|X^TS^TSX -
X^TX\|_F^2$ is small. Note that the matrix product $(X^TS^T)
\cdot (SX)$ can be performed more quickly if $S$ has a small number of
rows. Unfortunately, for the guarantees we want, all known
constructions of $S$ require $\Omega(N)$ rows. In fact, we prove an
information-theoretic result that any oblivious sketching matrix $S$
must have $\Omega(N/\log N)$ rows in Lemma~\ref{lem:noJL}, thus ruling out this
approach for achieving faster running time. Our proof uses arguments
from communication complexity, arguing that such a family of sketching
matrices would imply a better protocol for solving multiple copies of
the Gap-Hamming communication problem.

Another way of trying to improve the runtime is to give an
alternative definition of the problem: Instead of outputting a
$d \times d$ matrix that is close to $\Sigma$, the algorithm outputs a set of
nonnegative weights $w$ such that $\|w\|_1 = 1$, $\|w\|_{\infty} \leq
\frac{1}{(1-\epsilon)N}$, and $\|\sum_{i=1}^N w_i X_i X_i^T -
\Sigma\|_F = O(1)$. This bypasses the arguments above, since in the
case of no corruption, we do not have to actually output $X^TX$ and
can just set $w_i = 1/N$ for all $i$. However, even for this relaxed version of the problem, we
show that unless one can solve a certain ``column norm''
distinguishing problem faster than rectangular matrix multiplication,
one cannot solve this problem faster than $(d, d^2, d)$ matrix multiplication time.
This problem can be intuitively stated as follows: the good samples are drawn from $\NN(0,I)$,
and the corrupted samples are drawn from a mean-zero Gaussian distribution with a
very slight and known perturbation to the identity covariance matrix.
One needs to identify a large fraction of these corrupted
samples. Even though the covariance of the perturbed Gaussians is
known, it is so slight that the norms of the corrupted samples are
very similar to the uncorrupted ones. Therefore, one needs to measure
these norms along certain directions, which requires computing a matrix
product of the samples with a worst-case covariance matrix. We show
that outputting the weights described above requires solving this
problem, which we conjecture to be hard.


\subsection{Related and Prior Work} \label{ssec:prior}

Learning in the presence of outliers is an important goal in 
statistics and has been studied in the robust statistics community 
since the 1960s~(\cite{Huber64}). After several decades of work, a number of sample-efficient and 
robust estimators have been discovered. The reader is referred to~(\cite{DKKLMS16, LaiRV16}) 
for a detailed summary of this line of work. Until recently, all known computationally efficient high-dimensional estimators 
could only tolerate a negligible fraction of outliers. Recent work~(\cite{DKKLMS16, LaiRV16}) gave the first efficient robust
estimators for basic high-dimensional unsupervised tasks. Since these works, 
there has been a flurry of research activity on robust learning algorithms in both supervised
and unsupervised settings~(\cite{BDLS17, CSV17, DKK+17, DKS17-sq, DiakonikolasKKLMS18, SteinhardtCV18, DiakonikolasKS18-mixtures, DiakonikolasKS18-nasty, HopkinsL18, KothariSS18, PrasadSBR2018, DiakonikolasKKLSS2018sever, KlivansKM18, DKS19-lr, LSLC18-sparse, ChengDKS18}).

The most relevant prior work is that of~\cite{ChengDR19}, initiating
the direction of obtaining fast algorithms for robust high-dimensional estimation.
For the problem of robust mean estimation,~\cite{ChengDR19} 
proposed a primal-dual approach -- building on the convex programming approach of ~\cite{DKKLMS16} -- yielding an algorithm with runtime $\tilde{O}(N d)/\poly(\eps)$.
This improved on the $\tilde{O}(Nd^2)$ runtime of the iterative filtering method in~\cite{DKKLMS16}.
Our algorithm uses the same primal-dual framework; however, we emphasize that a standard application of their framework would only 
lead to a runtime of $\tilde{O}(Nd^{2})/\poly(\eps)$. To obtain our improved runtime, we need to overcome
a number of technical obstacles, as we explained in Section~\ref{sec:overview}.

\makeatletter{}
\section{Preliminaries}
\label{sec:prelim}

\paragraph{Basic Notations.}
For a positive integer $n$, we write $[n]$ for the set $\{1, \ldots, n\}$.
We use $e_i$ to denote the $i$-th standard basis vector, and $I$ to denote the identity matrix.
For a vector $x$, we use $\normone{x}$, $\normtwo{x}$, and $\norminf{x}$ to denote the $\ell_1$, $\ell_2$, and $\ell_\infty$ norm of $x$ respectively.
For a matrix $A$, we use $\normtwo{A}$ and $\fnorm{A}$ to denote the spectral norm and Frobenius norm of $A$ respectively.

Let $\tr(A)$ be the trace of $A$, and $\kappa(A)$ be the condition number of $A$.
A symmetric matrix $A \in \R^{n \times n}$ is said to be positive semidefinite (PSD) if $x^\top A x \ge 0$ for all $x \in \R^n$.
For two symmetric matrices $A$ and $B$, we write $A \preceq B$ iff the matrix $B - A$ is positive semidefinite.

For two vectors $x$ and $y$ of the same dimensions, let $\inner{x, y} = x^\top y = \sum_i x_i y_i$ be the inner product of $x$ and $y$.
For two matrices $A$ and $B$ of the same dimensions, let $A \bullet B = \inner{A, B} = \tr(A^\top B)$ be the entry-wise inner product of $A$ and $B$.
For a matrix $A \in \R^{n \times n}$, we use $A^\flat \in \R^{n^2}$ to denote its canonical flattening into a vector.

Throughout this paper, we use $D$ to denote the ground-truth distribution. We use $d$ for the dimension of $D$, $N$ for the number of samples, and $\eps$ for the fraction of corrupted samples.
We use $X \sim D$ to denote a sample (i.e., a vector random variable) drawn from $D$.
Given $N$ (possibly corrupted) samples $(X_i)_{i=1}^N$ drawn from $D$, we often abuse notation and again use $X$ to denote the $N \times d$ matrix where the $i$-th column of $X$ is the $i$-th sample $X_i \in \R^d$.

\paragraph{Connections Between the Second and Fourth Moments of a Gaussian.}
Let $A \otimes B$ denote the Kronecker product of $A$ and $B$.
In this paper, we frequently consider the Kronecker product of a sample $X \in \R^{d}$ with itself: $Z = X \otimes X \in \R^{d^2}$.
Our algorithms crucially rely on the following lemma, which characterizes the connections between the second-moment and fourth-moment tensor of a Gaussian.
Lemma~\ref{lem:tensor-cov} is proved in Appendix~\ref{app:prelim}.

\begin{lemma}
\label{lem:tensor-cov}
Let $X \sim \NN(0, \Sigma)$ and $Z = X \otimes X$.  Let $\Sigma_Z \in \R^{d^2 \times d^2}$ be the covariance matrix of $Z$.
We have (i) if $\Sigma \preceq I$, then $\cov(Z) \preceq 2 I$; and (ii) if $0 < \tau < 1$ and $\normtwo{\Sigma - I} \le \tau$, then $\normtwo{\cov(Z) - 2 I} \le 6 \tau$.
\end{lemma}

\paragraph{Fast Rectangular Matrix Multiplication.}
We will frequently use fast rectangular matrix multiplication in our algorithms.
Let $(a, b, c)$-matrix multiplication time denote the time it takes to multiply an $a \times b$ matrix with a $b \times c$ matrix.
It is forklore (see, e.g.,~\cite{LottiR83}) that $(a, a, b)$, $(a, b, a)$, and $(b, a, a)$ matrix multiplications require the same number of arithmetic operations.
More specifically, we use the algorithm proposed in \cite{Gall12}.
They obtained new upper bounds on $(n, n^\alpha, n)$ matrix multiplication time.
We use a special case of their result ($\alpha = 2$).

\begin{lemma}[Fast Rectangular Matrix Multiplication~(\cite{Gall12})]
\label{lem:fast-mult}
We can compute $(n, n^2, n)$-matrix multiplication in time  $O(n^{3.26})$.
This implies that for $d > 0$ and $N = \tilde O(d^2 / \eps^2)$, we can multiply a $d \times N$ matrix with an $N \times d$ matrix in time $\tilde O(d^{3.26}/\eps^2)$.
\end{lemma}
To multiply a $d \times N$ matrix with an $N \times d$ matrix when $N = \tilde O(d^2 / \eps^2)$, we split the matrices into blocks of size $d \times d^2$ or $d^2 \times d$, 
  multiply each pair of matrices and then add the results together.
The total running time is $\frac{N}{d^2} \cdot O(d^{3.26} + d^2) = \tilde O(d^{3.26}/\eps^2)$.

\paragraph{Transposition Principle for Matrix-Vector Multiplication.}
The transposition principle (see, e.g.,~\cite{Bordewijk57,Fiduccia73}) plays a central role in our faster implementation of positive SDP solvers.
It states that matrix-vector multiplication by $A^\top$ has almost exactly the same computational complexity as matrix-vector multiplication by $A$.
\begin{lemma}[Transposition Principle~(\cite{Fiduccia73})]
\label{lem:lr-mult}
Fix a matrix $A \in \R^{r \times c}$.
Suppose there exists an arithmetic circuit of size $s$ that can compute $Ax$ for arbitrary $x \in \R^c$.
Then, there exists an arithmetic circuit of size $O(s + m)$ that computes $A^\top y$ for arbitrary $y \in \R^r$.
\end{lemma}

\makeatletter{}
\section{Estimating the Covariance of a Gaussian Distribution}
\label{sec:cov}

In this section, we present our key structural and computational lemmas, and use them 
to prove our main algorithmic results (Theorems~\ref{thm:main}~and~\ref{thm:main-additive}).

\subsection{Robust Covariance Estimation: Multiplicative Approximation}
\label{sec:main}
We first present our algorithm (Algorithm~\ref{alg:main}) for robustly estimating the covariance of Gaussian distributions with multiplicative error guarantees.
Algorithm~\ref{alg:main} starts with an upper bound $\Sigma_0$ on the true covariance matrix $\Sigma$,
  and iteratively compute more and more accurate upper bounds $\Sigma_{t} \succeq \Sigma$.

\begin{algorithm}[h]
  \caption{Robust Covariance Estimation for Gaussian Distributions (Multiplicative Error)}
  \label{alg:main}
  \SetKwInOut{Input}{Input}
  \SetKwInOut{Output}{Output}
  \Input{$0 < \eps < \eps_0$, and an $\eps$-corrupted set of $N = \tilde \Omega(d^2/\eps^2)$ samples drawn from $\NN(0, \Sigma)$.}
  \Output{A matrix $\hat \Sigma \in \R^{d \times d}$ such that with high prob., $\fnorm{\hat \Sigma^{-1/2} \Sigma \hat \Sigma^{-1/2} - I} \le O(\eps\log(1/\eps))$.}
  Compute an initial upper bound $\Sigma_0$ with $\Sigma \preceq \Sigma_0 \preceq \kappa \poly(d) \Sigma$ using Lemma~\ref{lem:start}. \\
  Let $T_1 = O(\log \kappa + \log d)$ and $T_2 = T_1 + O(\log\log(1/\eps))$. \\
  \For{$t = 0$ {\bf to} $T_1 - 1$}{
    Compute $\Sigma_{t+1} \in \R^{d \times d}$ with $\Sigma \preceq \Sigma_{t+1} \preceq \Sigma + O(\sqrt{\eps}) \Sigma_t$ using Lemma~\ref{lem:first-loop}  on $\Sigma_t$.
  }
  Let $\tau_{T_1} = O(\sqrt{\eps})$. \\
  \For{$t = T_1$ {\bf to} $T_2 - 1$}{
   Let $\tau_{t+1} = O(\sqrt{\eps \tau_t} + \eps \log(1/\eps))$. \\
    Compute $\Sigma_{t+1} \in \R^{d \times d}$ with $\Sigma \preceq \Sigma_{t+1} \preceq \Sigma + \tau_{t+1} \Sigma_t$ using Lemma~\ref{lem:second-loop}  on $\Sigma_t$.
  }
  Compute $\hat\Sigma$ by invoking Lemma~\ref{lem:second-loop} on $\Sigma_{T_2}$. \\
  \Return {$\hat \Sigma$.}
\end{algorithm}

First we need a reasonable starting point before we can run any iterative refinement steps.
\begin{lemma}
\label{lem:start}
Consider the same setting as in Theorem~\ref{thm:main}.
We can compute a matrix $\Sigma_0$ in $\tilde O(d^{3.26}/\eps^2)$ time such that, with high probability, $\Sigma \preceq \Sigma_0 \preceq (\kappa \poly(d)) \Sigma$ and $\normtwo{\Sigma_0} \le \poly(d) \normtwo{\Sigma}$.
\end{lemma}

We use two different iterative refinement steps (Lemmas~\ref{lem:first-loop}~and~\ref{lem:second-loop}), which correspond to the two loops in Algorithm~\ref{alg:main}.
In the first phase (Lemma~\ref{lem:first-loop}), we only have a crude upper bound on $\Sigma$.

\begin{lemma}[Iterative Refinement: Getting $\sqrt{\eps}$ Error]
\label{lem:first-loop}
Consider the same setting as in Theorem~\ref{thm:main}.
Given an upper bound $\Sigma_t \in \R^{d \times d}$ on the unknown covariance matrix $\Sigma$, i.e., $\Sigma \preceq \Sigma_t$,
  we can compute in time $\tilde O(d^{3.26} / \eps^8)$ an upper bound matrix $\Sigma_{t+1} \in \R^{d \times d}$, and a hypothesis matrix $\hat\Sigma \in \R^{d \times d}$ such that, with high probability,
\[ \Sigma \preceq \Sigma_{t+1} \preceq \Sigma + O(\sqrt{\eps}) \Sigma_t \; , \text{ and} \]
\[ \|\hat\Sigma - \Sigma\|_F \le O(\sqrt{\eps}) \normtwo{\Sigma_t} \; . \]
\end{lemma}

The first phase can only converge to a matrix $\Sigma_{T_1}$ with $\Sigma \preceq \Sigma_{T_1} \preceq (1 + O(\sqrt{\eps}) \Sigma$.
In the second phase (Lemma~\ref{lem:second-loop}), we already have a fairly accurate estimate of $\Sigma$, so the refinement steps converge faster and eventually we can get to a matrix $\Sigma_{T_2}$ with $\Sigma \preceq \Sigma_{T_2} \preceq (1 + O(\eps \log(1/\eps)) \Sigma$.

\begin{lemma}[Iterative Refinement: Getting $\eps \log(1/\eps)$ Error]
\label{lem:second-loop}
Consider the same setting as in Theorem~\ref{thm:main}.
Let $0 < \tau_t < \tau_0$ for some universal constant $\tau_0$.
Given $\tau$ and $\Sigma_t$ with $\Sigma \preceq \Sigma_t \preceq (1 + \tau_t) \Sigma$,
  we can compute in time $\tilde O(d^{3.26} / \eps^8)$ an upper bound matrix $\Sigma_{t+1}$ and a hypothesis matrix $\hat\Sigma$ such that, with high probability, for $\tau_{t+1} = O(\sqrt{\eps \tau} + \eps \log(1/\eps))$,
\[  \Sigma \preceq \Sigma_{t+1} \preceq \Sigma + \tau_{t+1} \Sigma_t \; , \text{ and} \]
\[ \| \Sigma^{-1/2} \hat \Sigma \Sigma^{-1/2} - I \|_F \le \tau_{t+1} \; .\]
\end{lemma}

We defer the proof of Lemma~\ref{lem:start} to Appendix~\ref{app:cov}, and the proofs of Lemmas~\ref{lem:first-loop}~and~\ref{lem:second-loop} to Section~\ref{sec:iter}.
We first use these three lemmas to prove Theorem~\ref{thm:main} (correctness and runtime of Algorithm~\ref{alg:main}).

\bigskip
\begin{proof}[Proof of Theorem~\ref{thm:main}]
We first use Lemma~\ref{lem:start} to find an upper bound $\Sigma_0 \in \R^{d \times d}$ on the true covariance matrix $\Sigma$ such that $\Sigma \preceq \Sigma_0 \preceq (\kappa \poly(d)) \Sigma$.

For any integer $t \ge 0$, given the upper bound matrix $\Sigma_t$, we can use Lemma~\ref{lem:first-loop} to obtain a better upper bound $\Sigma_{t+1}$ such that
$\Sigma_{t+1} \preceq \Sigma + O(\sqrt{\eps}) \Sigma_t$.
Since $\Sigma_0 \preceq \kappa \poly(d) \Sigma$, after $T_1 = O(\log \kappa + \log d)$ iterations, we have a matrix $\Sigma_{T_1}$ with
$\Sigma \preceq \Sigma_{T_1} \preceq \left(1 + O(\sqrt{\eps})\right)\Sigma$. 

At this point we have a pretty accurate upper bound $\Sigma_{T_1}$ with $\Sigma \preceq \Sigma_{T_1} \preceq (1 + \tau_{T_1})\Sigma$,
  where $\tau_{T_1} = O(\sqrt{\eps})$.
For any integer $t \ge T_1$, given $\Sigma_t$ and $\tau_t$, we can use Lemma~\ref{lem:second-loop} to obtain a better upper bound matrix $\Sigma_{t+1}$ such that
  $\Sigma \preceq \Sigma_{t+1} \preceq \Sigma + \tau_{t+1} \Sigma_t$, where $\tau_{t+1} = O(\sqrt{\eps \tau_t} + \eps \log(1/\eps))$.
Similar to the previous step, after $O(\log \log(1/\eps))$ iterations,
  we have a matrix $\Sigma_{T_2}$ such that
$\Sigma \preceq \Sigma_{T_2} \preceq (1 + \tau_{T_2})\Sigma$, where $\tau_{T_2} = O(\eps \log(1/\eps))$.

Finally, using Lemma~\ref{lem:second-loop} one more time with $\Sigma_{T_2}$ and $\tau_{T_2}$,
  we can get a matrix $\hat \Sigma$ with
\[ \fnorm{\Sigma^{-1/2} \hat\Sigma \Sigma^{-1/2} - I} = O(\sqrt{\eps \tau_{T_2}} + \eps \log(1/\eps)) = O(\eps \log(1/\eps)) \; . \]
We note that both Lemmas~\ref{lem:first-loop}~and~\ref{lem:second-loop} hold with high probability, so we can take a union bound over the failure probabilities and conclude that with high probability all iterative refinement steps are successful, and therefore, we can return $\hat \Sigma$ as our final answer.

Now we analyze the running time of Algorithm~\ref{alg:main}.
We call Lemma~\ref{lem:start} once to compute $\Sigma_0$, which takes time $\tilde O(d^{3.26}/\eps^2)$.
After that, we use two iterative refinement steps. The total number of iterations is $O(\log\kappa + \log d + \log\log(1/\eps))$.
In each iteration, we invoke either Lemma~\ref{lem:first-loop}~or~\ref{lem:second-loop}.
Since both lemmas run in time $\tilde O(d^{3.26}/\eps^8)$, the overall running time is
$\tilde O(d^{3.26}/\eps^2) + O(\log\kappa + \log d + \log\log(1/\eps)) \cdot \tilde O(d^{3.26}/\eps^8) = \tilde O(d^{3.26}/\eps^8)$.
\end{proof}

\subsection{Implementing the Iterative Refinement Steps}
\label{sec:iter}

In this section, we prove the iterative refinement lemmas (Lemmas~\ref{lem:first-loop}~and~\ref{lem:second-loop}).
Both refinement steps use robust mean estimation algorithms as subroutines.
More specifically, let $Y_i = \Sigma_t^{-1/2} X_i$ and $Z_i = Y_i \otimes Y_i$.
When $X \sim \NN(0, \Sigma)$, the covariance matrix of $Y$ is $\Sigma_Y = \Sigma_t^{-1/2} \Sigma \Sigma_t^{-1/2}$,
  so the mean of $Z$ is $\expect{}{Z} = \left(\Sigma_t^{-1/2} \Sigma \Sigma_t^{-1/2}\right)^\flat$.
If we can get a good estimate of $\expect{}{Z}$, we can use this information to obtain a better upper bound $\Sigma_{t+1}$.

In the first phase (Lemma~\ref{lem:first-loop}), we only have a crude upper bound on $\Sigma$.
For any $\Sigma_t \succeq \Sigma$, we have $\Sigma_Y = \Sigma_t^{-1/2} \Sigma \Sigma_t^{-1/2} \preceq I$, which implies $\Sigma_Z \preceq 2I$ (Lemma~\ref{lem:tensor-cov}).
Because $Z$ has bounded covariance, we can use the following robust mean estimation algorithm from~\cite{ChengDR19}.

\begin{lemma}[Robust Mean Estimation for Bounded-Covariance Distributions, \cite{ChengDR19}]
\label{lem:mean-boundedcov}
Let $D$ be a distribution on supported on $\R^d$ with unknown mean and unknown covariance matrix $\Sigma$ such that $\Sigma \preceq \sigma^2 I$.
Let $0 < \eps < \eps_0$ for some universal constant $\eps_0$, and let $\delta = O(\sqrt{\eps})$.
Given an $\eps$-corrupted set of $N = \tilde \Omega(d / \eps)$ samples drawn from $D$,
  Algorithm~\ref{alg:boundedcov} outputs a hypothesis vector $\hat \mu$ such that with high probability, $\normtwo{\mu - \mus} \le O(\sigma \delta) = O(\sigma \sqrt{\eps})$.
\end{lemma}

In the second phase (Lemma~\ref{lem:second-loop}), we use the fact that the current estimate $\Sigma_t$ is already very close to $\Sigma$, therefore $Y = \Sigma_t^{-1/2} X$ is very close to $\NN(0,I)$.
In this case we need an algorithm with stronger robust estimation guarantees tailored for the specific distribution of $Z = Y\otimes Y$.

\begin{lemma}[Robust Mean Estimation with Approximately Known Covariance]
\label{lem:mean-apx-cov}
Let $D$ be a distribution supported on $\R^d$ with unknown mean $\mus$ and covariance $\Sigma$.
Let $0 < \eps < \eps_0$ for some universal constant $\eps_0$, $\tau \le O(\sqrt{\eps})$, and $\delta = O(\sqrt{\tau\eps} + \eps \log(1/\eps))$.
Suppose that $D$ has exponentially decaying tails, and $\Sigma$ is close to the identity matrix $\normtwo{\Sigma - I} \le \tau$.
Given an $\eps$-corrupted set of $N = \tilde \Omega(d / \delta^2)$ samples drawn from $D$,
  Algorithm~\ref{alg:apxcov} outputs a hypothesis vector $\hat \mu$ such that with high probability, $\normtwo{\mu - \mus} \le O(\delta)$.
\end{lemma}

It is worth noting that we cannot use these mean estimation algorithms in a black-box manner.
This is because writing down the input explicitly takes $\Omega(N d^2) = \tilde \Omega(d^4/\eps^2)$ time,
  and these algorithms run in time $\tilde \Omega(N d^2) / \poly(\eps)$ in $d^2$ dimensions.
One of our main contributions is to show that it is possible to open up these algorithms and take advantage of the additional structures of our inputs (they all have the form $Y_i \otimes Y_i$) to implement both algorithms to run in time $\tilde O(d^{3.26})/\poly(\eps)$.

\begin{proposition}[Robust Mean Estimation with Tensor Input]
\label{prop:runtime-tensor}
If all input samples $(Z_i)_{i=1}^N$ have the form $Z_i = Y_i \otimes Y_i$ for some $Y_i \in \R^d$,
  and they are given implicitly as the vectors $(X_i)_{i=1}^N$,
  then both Algorithms~\ref{alg:boundedcov}~and~\ref{alg:apxcov} can be implemented to run in $\tilde O(d^{3.26}/\eps^8)$ time.
\end{proposition}

We give a description of the algorithm for Lemma~\ref{lem:mean-boundedcov} in Appendix~\ref{app:boundedcov} (Algorithm~\ref{alg:boundedcov}).
We prove Lemma~\ref{lem:mean-apx-cov} and present the corresponding algorithm (Algorithm~\ref{alg:apxcov}) in Section~\ref{sec:mean}.
In Section~\ref{sec:sdp}, we show that both algorithms can be implemented to run in time $\tilde O(d^{3.26})/\poly(\eps)$ (Proposition~\ref{prop:runtime-tensor}).
We first use Lemmas~\ref{lem:mean-boundedcov},~\ref{lem:mean-apx-cov},~and~Proposition~\ref{prop:runtime-tensor} to prove the iterative refinement lemmas.

\bigskip
\begin{proof}[Proof of Lemma~\ref{lem:first-loop}]
Given an upper bound $\Sigma_t$ on the true covariance matrix,
  we can rotate the input samples to compute $Y_i = \Sigma_t^{-1/2} X_i$.
Let $Z_i = Y_i \otimes Y_i$.
Note that when $X \sim \NN(0, \Sigma)$, the random variable $Y = \Sigma_t^{-1/2} X$ is drawn from a Gaussian distribution with covariance $\Sigma_Y = \Sigma_t^{-1/2} \Sigma \Sigma_t^{-1/2} \preceq I$.
Lemma~\ref{lem:tensor-cov} implies that, if $\Sigma_Y \preceq I$, then $\Sigma_Z \preceq 2 I$.
Therefore, $(Z_i)_{i=1}^N$ is an $\eps$-corrupted set of samples drawn from a distribution with bounded covariance, so we can apply Algorithm~\ref{alg:boundedcov} to robustly estimate its mean.

Let $M$ be the output of Algorithm~\ref{alg:boundedcov} reshaped into a $d \times d$ matrix.
Because
$ \expect{}{Z} = \left(\expect{}{Y Y^\top}\right)^\flat = \left(\Sigma_t^{-1/2} \Sigma \Sigma_t^{-1/2}\right)^\flat$
  and Algorithm~\ref{alg:boundedcov} (Lemma~\ref{lem:mean-boundedcov}) guarantees that $\normtwo{M^\flat - \expect{}{Z}} \le O(\sqrt{\eps})$,
\[ \fnorm{M - \Sigma_t^{-1/2} \Sigma \Sigma_t^{-1/2}} \le O(\sqrt{\eps}) \; . \]
Let $\hat\Sigma = \Sigma_t^{1/2} M \Sigma_t^{1/2}$.  Using $\fnorm{A B} \le \normtwo{A} \fnorm{B}$, we can prove the first part of the lemma,
\[ \fnorm{\hat\Sigma - \Sigma} = \fnorm{\Sigma_t^{1/2} M \Sigma_t^{1/2} - \Sigma} \le O(\sqrt{\eps})\normtwo{\Sigma_t} \; . \]

As a result, we have that $-O(\sqrt{\eps}) I \preceq \Sigma_t^{-1/2} (\hat \Sigma - \Sigma) \Sigma_t^{-1/2} \preceq O(\sqrt{\eps}) I$, or equivalently
\[ \hat\Sigma - O(\sqrt{\eps}) \Sigma_t \preceq \Sigma \preceq \hat\Sigma + O(\sqrt{\eps}) \Sigma_t \; . \]
Now $\Sigma_{t+1} = \hat\Sigma + O(\sqrt{\eps}) \Sigma_t$ is a better upper bound, which satisfies $\Sigma \preceq \Sigma_{t+1} \preceq \Sigma + O(\sqrt{\eps})\Sigma_t$.

For the running time, we can compute $\Sigma_t^{-1/2}$ explicitly in time $O(d^\omega)$ using SVD~\cite{demmel2007fast}.
Given the input sample matrix $X \in \R^{d \times N}$,
  we can apply $\Sigma_t^{-1/2}$ to all samples by computing $Y = \Sigma_t^{-1/2} X$ via fast rectangular matrix multiplication in time $\tilde O(d^{3.26}/\eps^2)$ (Lemma~\ref{lem:fast-mult}).

Since all the $Z_i$'s have the form $Y_i \otimes Y_i$,
Proposition~\ref{prop:runtime-tensor} shows that Algorithms~\ref{alg:boundedcov} has running time $\tilde O(d^{3.26}/\eps^8)$.
Given the output of Algorithms~\ref{alg:boundedcov}, we can compute the new upper bound $\Sigma_{t+1}$ in time $O(d^\omega)$ using a constant number of $d \times d$ matrix additions and multiplications.
\end{proof}

\begin{proof}[Proof of Lemma~\ref{lem:second-loop}]
Let $Y_i = \Sigma_t^{-1/2} X_i$ and $Z_i = Y_i \otimes Y_i$.
We know that
\[ \normtwo{\Sigma_Y - I} = \normtwo{\Sigma_t^{-1/2} \Sigma \Sigma_t^{-1/2} - I} \le \tau \; . \]
By Lemma~\ref{lem:tensor-cov}, we have $\normtwo{\Sigma_Z - 2 I} \le O(\tau)$.
By standard concentration results, $Z = Y \otimes Y$ has exponential concentration about its mean in any direction.
Therefore, $(Z_i)_{i=1}^N$ is an $\eps$-corrupted set of samples drawn from a distribution that satisfies the conditions in Lemma~\ref{lem:mean-apx-cov},
  and we can apply Algorithm~\ref{alg:apxcov} to robustly learn the mean of $Z$.
By Lemma~\ref{lem:mean-apx-cov}, we can compute a matrix $M$ such that
\[ \fnorm{M - \Sigma_t^{-1/2} \Sigma \Sigma_t^{-1/2}} \le O(\sqrt{\eps \tau} + \eps \log(1/\eps)) \; . \]
Let $\hat\Sigma = \Sigma_t^{1/2} M \Sigma_t^{1/2}$.
Using $\fnorm{A B} \le \normtwo{A} \fnorm{B}$, we can prove the first part of the lemma,
\begin{align*}
 \fnorm{\Sigma^{-1/2} \hat\Sigma \Sigma^{-1/2} - I}
&=\fnorm{\Sigma^{-1/2} \Sigma_t^{1/2} M \Sigma_t^{1/2} \Sigma^{-1/2} - I} \\
&\le \normtwo{\Sigma^{-1/2} \Sigma_t^{1/2}} \fnorm{M - \Sigma_t^{-1/2} \Sigma \Sigma_t^{-1/2}} \normtwo{\Sigma_t^{1/2} \Sigma^{-1/2}} \\
&\le (1 + \tau) \cdot O(\sqrt{\eps \tau} + \eps \log(1/\eps))  \\
&= O(\sqrt{\eps \tau} + \eps \log(1/\eps)) \; .
\end{align*}
This gives a better upper bound $\Sigma_{t+1} = \hat\Sigma + O(\sqrt{\eps \tau} + \eps \log(1/\eps)) \Sigma_t$ such that
  $\Sigma \preceq \Sigma_{t+1} \preceq \Sigma + O(\sqrt{\eps \tau} + \eps \log(1/\eps)) \Sigma_t$.
  
We omit the running time analysis because it is identical to the one in the previous proof.
\end{proof}

\section{Robust Mean Estimation Subroutines}
\label{sec:mean}
We first present the robust mean estimation algorithm (Algorithm~\ref{alg:apxcov}) that achieves Lemma~\ref{lem:mean-apx-cov}. 
\begin{algorithm}[h]
  \caption{Robust Mean Estimation with Approximately Known Covariance}
  \label{alg:apxcov}
  \SetKwInOut{Input}{Input}
  \SetKwInOut{Output}{Output}
  \Input{An $\eps$-corrupted set of $N$ samples $\{Z_i\}_{i=1}^N$ on $\R^d$ with $N = \tilde \Omega(d/\eps^2)$ and $\eps < \eps_0$.}
  \Output{A vector $\hat \mu \in \R^d$ such that, with high probability, $\normtwo{\hat \mu - \mus} \le O(\sqrt{\eps \tau} + \eps\log(1/\eps))$.}
  Let $\nu \in \R^d$ be an initial guess with $\normtwo{\nu - \mus} \le \poly(d)$. \\
  \For{$i = 1$ {\bf to} $O(\log d)$}{
   Compute a near-optimal solution $w \in \R^N$ to the primal SDP~\eqref{eqn:primal-sdp} with parameters $\nu$ and $2\eps$. \\
   Compute a near-optimal solution $M \in \R^{d \times d}$ for the dual SDP~\eqref{eqn:dual-sdp} with parameters $\nu$ and $\eps$.\\
   \eIf{the value of $w$ in SDP~\eqref{eqn:primal-sdp} is at most $1 + c (\tau + \eps\log^2(1/\eps))$ for a universal constant $c$}{
     \Return{the weighted empirical mean $\hat \mu_w = \sum_{i=1}^N w_i Z_i$} (Lemma~\ref{lem:wrong-mean-primal-nosol}) .\\
   }{
     Move $\nu$ closer to $\mus$ using the top eigenvector of $M$ (Lemma~\ref{lem:good-dual-better-nu}).
   }
  }
\end{algorithm}

We first recall the primal-dual approach recently developed in~\cite{ChengDR19}.
Given data points $Z_i = X_i\otimes X_i$, their algorithm starts with a guess $\nu \in \R^{d^2}$, and then in each iteration solves the primal and dual SDPs~\eqref{eqn:primal-sdp}~and~\eqref{eqn:dual-sdp}.
They showed that either a good primal solution gives weights $w \in \R^N$ such that the weighted average $\sum_{i=1}^N w_i Z_i$ is close to the true mean;
  or a good dual solution must identify a direction of improvement that allows the algorithm to move $\nu' \in \R^{d^2}$ much closer to the true mean of $Z$.

There are two obstacles for applying the algorithmic framework of~\cite{ChengDR19} to our setting.
First, the input samples $Z_i = X_i \otimes X_i$'s are $d^2$-dimensional vectors.
Writing down these vectors explicitly takes time $\Omega(Nd^2) = \Omega(d^4)$.
Therefore, we want to solve the SDPs~\eqref{eqn:primal-sdp}~and~\eqref{eqn:dual-sdp} on input $Z_i$ without computing them explicitly.
We resolve this issue in Section~\ref{sec:sdp} (Proposition~\ref{prop:runtime-tensor}).

Second, their algorithms have error $O(\sqrt{\eps})$ for bounded-covariance distributions, and error $O(\eps\sqrt{\log(1/\eps)})$ for sub-gaussian distributions with identity covariance matrix.
While we can directly use their result for bounded-covariance distributions for Lemma~\ref{lem:mean-boundedcov}, we need to develop a new algorithm for Lemma~\ref{lem:mean-apx-cov}.
In Lemma~\ref{lem:mean-apx-cov}, we have a distribution with exponential decaying tails, and we know its covariance is $\tau$-close to the identity matrix.
We want to robustly estimate its mean, with optimal error guarantees that depend on both $\eps$ and $\tau$.
We generalize the analysis of~\cite{ChengDR19} to handle this case.
Lemma~\ref{lem:mean-apx-cov} is proved in Appendix~\ref{app:mean}.

\section{Faster Implementation of Robust Mean Estimation with Tensor Inputs}
\label{sec:sdp}
The bottleneck of both Algorithms~\ref{alg:boundedcov}~and~\ref{alg:apxcov} are solving SDPs~\eqref{eqn:primal-sdp}~and~\eqref{eqn:dual-sdp}.
In this section, we prove Proposition~\ref{prop:runtime-tensor}, which states that when all input samples have the tensor-product form $Y \otimes Y$, we can solve these SDPs in time $\tilde O(d^{3.26})/\poly(\eps)$.

We first convert the SDPs~\eqref{eqn:primal-sdp}~and~\eqref{eqn:dual-sdp} into packing/covering SDPs as follows.

\begin{lp}
\label{eqn:primal-standard}
\maxi{\mone^\top w}
\st \con{w_i \ge 0, \sum_{i=1}^N w_i A_i \preceq I}
\end{lp}
\begin{lp}
\label{eqn:dual-standard}
\mini{\tr(M)}
\st \con{A_i \bullet M \ge 1, M \succeq 0 \;,}  
\end{lp}
where each $A_i \in \R^{(d^2+N)\times(d^2+N)}$ is a PSD matrix given by\footnote{Recall that $X_i \in \R^{d \times 1}$ is the $i$-th sample, and $e_i \in \R^{N \times 1}$ is the $i$-th standard basis vector.}
\begin{align}
\label{eqn:pc-sdp-input}
A_i & = \left[ \begin{array}{cc}
\rho (Z_i-\nu) (Z_i-\nu)^\top & 0 \\ 0 & (1-\eps)N \cdot e_i e_i^\top
\end{array} \right] \; .
\end{align}

Here $\rho$ is a binary search parameter that is between $\frac{1}{d}$ and $1$.
At the core of nearly-linear time width-independent SDP solvers~(e.g., \cite{AllenLO16,PengTZ16}) is an application of matrix multiplicative weight update,
  where the algorithm maintains a weighted sum $\Psi$ of the matrices.
In iteration $t$, we have $\Psi^t = \sum_{i=1}^n w_i A_i$, and we will update the weights based on the values of $A_i \bullet \frac{\exp(\Psi^t)}{\tr(\exp(\Psi^t))}$.

\begin{lemma}[Positive SDP Solver,~\cite{PengTZ16}]
\label{lem:ptz16}
Let $A_1, \ldots, A_n$ be $m \times m$ PSD matrices given in factorized form $A_i = C_i C_i^\top$.
Consider the following pair of packing and covering SDPs:
\begin{alignat*}{3}
\max_{x \ge 0}     & \; \mone^\top x && \quad \textup{ s.t. } \sum_{i=1}^n x_i A_i \preceq I \; . \\
\max_{Y \succeq 0} & \; \tr(Y)       && \quad \textup{ s.t. } A_i \bullet Y \ge 1, \forall i \; .
\end{alignat*}
Fix $\eps > 0$.
Given an oracle algorithm that, on input $\Psi = \sum_{i=1}^n w_i A_i$ with $\normtwo{\Psi} = O(\log(n)/\eps)$,
  runs in time $T_{\exp}$ and returns $(1\pm \eps)$-multiplicative approximations to $\frac{\exp(\Psi)}{\tr(\exp(\Psi))} \bullet A_i$ for all $i$.
Then, we can compute feasible primal and dual solutions $x$ and $Y$, such that with high probability, $\mone^\top x \ge (1-O(\eps)) \OPT$ and $\tr(Y) \le (1+O(\eps)) \OPT$.
Moreover, we can do so in time $\tilde O((T_{\exp} + n) \log^2 n / \eps^{3})$, where $q$ is the total number of non-zero entries in the $C_i$'s.
\end{lemma}

In the rest of this section, we will prove that when each $Z_i$ has the form $Z_i = Y_i \otimes Y_i$,
  we can implement the oracle algorithm required by Lemma~\ref{lem:ptz16} in time $T_{\exp} = \tilde O(d^{3.26} / \eps^5)$.
It is worth pointing out that we need to implement this oracle without ever writing down $\Psi \in \R^{d^2 \times d^2}$ explicitly.
  
We will approximate each $\exp(\Psi) \bullet A_i$ and $\tr(\exp(\Psi))$ separately.
Observe that $\Psi = \sum_{i=1}^n w_i A_i$ and $\exp(\Psi)$ have the same block structure as the $A_i$'s.
Due to the special structure of the bottom-right block, we can compute its contribution to $\tr(\exp(\Psi))$ and $\exp(\Psi) \bullet A_i$ exactly.
Therefore, we can focus on the top-left block.
Moreover, because the goal is to compute a multiplicative approximation of the top-left block's contribution to $\tr(\exp(\Psi))$ and $\exp(\Psi) \bullet A_i$,
  we can ignore the scalar $\rho$.
We prove the following lemma.

\begin{lemma}
\label{lem:mat-exp}
Fix $d > 0$ and $N = \tilde \Omega(d^2 / \eps^2)$.
Fix $Y \in \R^{d \times N}$, $w \in \R^N$, $\nu \in \R^{d^2}$, and $0 < \eps < 1$.
Let $Z_i = Y_i \otimes Y_i$ and (abusing notation) let $\Psi = \sum_{i=1}^N w_i (Z_i - \nu)(Z_i - \nu)^\top$.
Suppose $\normtwo{\Psi} = O(\log d / \eps)$.
We can compute, in time $\tilde O(d^{3.26}/\eps^5)$, $(1\pm \eps)$-multiplicative approximations to $\tr(\exp(\Psi))$ and $\exp(\Psi) \bullet ((Z_i - \nu) (Z_i - \nu)^\top)$ for all $i \in [N]$.
\end{lemma}
\begin{proof}
Let $Z \in \R^{d \times N}$ be the matrix whose $i$-th column is $Z_i$.
Observe that
$\Psi = (Z - \nu {\bf 1}^\top) D_w (Z - \nu {\bf 1}^\top)^\top \in \R^{d^2 \times d^2}$,
  where $D_w \in \R^{N \times N}$ is a diagonal matrix with $w$ on the diagonal.
First we show that, for any vector $s \in \R^{d^2}$, we can compute the matrix-vector multiplication
$\Psi s = (Z - \nu {\bf 1}^\top) D_w (Z - \nu {\bf 1}^\top)^\top s$
  in time $\tilde O(d^{3.26}/\eps^2)$.
This is because
\begin{enumerate}[itemsep=0pt,parsep=0pt,topsep=2pt]
\item[(i)] $(Z - \nu {\bf 1}^\top) s = Zs - ({\bf 1}^\top s) \nu$ and we can compute $Z s = \left(Y D_s Y^\top\right)^\flat$ via fast rectangular matrix multiplication (Lemma~\ref{lem:fast-mult}),
\item[(ii)] matrix-vector multiplication with $(Z - \nu {\bf 1}^\top)$ or $(Z - \nu {\bf 1}^\top)^\top$ has the same running time by the Transposition Principle (Lemma~\ref{lem:lr-mult} in Section~\ref{sec:prelim})), and
\item[(iii)] multiplication with a diagonal matrix $D_w$ can be done in time $O(N)$.
\end{enumerate}

We continue to show how to compute some $\eta$ such that $\eta \approx_{\eps/2} \tr(\exp(\Psi))$.
Since $\normtwo{\Psi} = \eps^{-1}\log d$, by Lemma~\ref{lem:taylor}, we can find a degree-$(\eps^{-1}\log d)$ matrix polynomial $p$ such that $p(\Psi) \approx_{\eps/8} \exp(\Psi/2)$.
Let $M = p(\Psi)$, we have $\tr(M^2) \approx_{\eps/4} \tr(\exp(\Psi))$.
Thus, it is sufficient to compute some $\eta \approx_{\eps/4} \tr(M^2)$.
We will write $\tr(M^2) = \sum_{i=1}^{d^2} (M^2)_{i,i}$ and approximate all $(M^2)_{i,i} = \normtwo{M e_i}^2$ simultaneously.
By the Johnson-Lindenstrauss lemma, there is a $O(\log d/\eps^2) \times d^2$ matrix $Q$ such that with high probability,
$\normtwo{M e_i}^2 \approx_{\eps/4} \normtwo{Q M e_i}$
  for all $i \in [d^2]$.
Note that $\Psi$ is symmetric and so is $M$.
We can compute $Q M = (M Q^\top)^\top$ by multiplying each column of $Q^\top$ through $M$,
  and each $M Q_i = p(\Psi) Q_i$ can be evaluated using $\deg(M)$ matrix-vector multiplications $\Psi v$ for some $v \in \R^{d^2}$.
The overall running time to approximate $\tr(\exp(\Psi))$ is
  $O(\log d / \eps^2) \cdot O(\log d / \eps) \cdot \tilde O(N + d^{3.26}/\eps^2) = \tilde O(d^{3.26}/\eps^5)$.

We approximate $\exp(\Psi) \bullet (Z_i - \nu)(Z_i - \nu)^\top$ using a similar approach:
$\exp(\Psi) \bullet (Z_i - \nu)(Z_i - \nu)^\top
  = \normtwo{\exp(\Psi/2) (Z_i - \nu)}^2
  \approx_{\eps/4} \normtwo{M (Z_i - \nu)}^2
  \approx_{\eps/4} \normtwo{Q M (Z_i - \nu)}^2$.
Notice that the last line is precisely the squared norm of the $i$-th column of $Q M (Z - \nu{\bf 1}^\top)$.
For the same reasons as in the previous case, we can compute this matrix in time $\tilde O(d^{3.26}/\eps^5)$.
\end{proof}

We can approximate $\exp(A)$ with a matrix polynomial of $A$, whose degree depends on the spectral norm of $A$ and the desired precision (see, e.g.,~\cite{AroraK16}).
\begin{lemma}[Taylor Expansion of Matrix Exponential]
\label{lem:taylor}
Let $A$ be PSD matrix with $\normtwo{A} \le \ell$, then there exists a polynomial $p(A)$ of degree $O(\max(\ell, \log(2/\eps)))$ such that $p(A) \approx_{\eps} \exp(A)$.
\end{lemma}

\begin{proof}[Proof of Proposition~\ref{prop:runtime-tensor}]
By Lemma~\ref{lem:ptz16}, we only need to show that the required oracle algorithm can be implemented in time $\tilde O(d^{3.26}/\eps^5)$.
We approximate $\tr(\exp(\Psi))$ and each $\exp(\Psi) \bullet A_i$ separately.
Given $\Psi = \sum_{i=1}^N w_i A_i$, we will compute the contribution from bottom-right block explicitly,
  and use Lemma~\ref{lem:mat-exp} for the top-left block.
The bottom-right block adds $\sum_{i=1}^N \exp(w_i (1-\eps) N)$ to $\tr(\exp(\Psi))$,
  and for every $i$, it adds $w_i \exp(w_i (1-\eps) N) $ to $\exp(\Psi) \bullet A_i$.
\end{proof}

\section{Evidence of Hardness}
\label{sec:hardness}
In this section, we provide some evidence which suggests that the running time of our algorithm has near-optimal dependence on $d$.
We start by noting that our sample complexity $N = \tilde \Omega(d^2/\eps^2)$ is tight up to polylogarithmic factors, and this holds even when there is no corruption.
For the rest of this section, we will assume both $\eps$ and $\kappa$ are constants, and focus on the dependence on $d$ in the running time.
Since the running time of our algorithm is dominated by $(d, d^2, d)$-matrix multiplication time,   faster matrix multiplication algorithms time will improve our running time.

In Section~\ref{sec:hardness-matrix}, we show that even when there are no corrupted samples, it is not known how to compute the empirical covariance matrix faster than $(d, d^2, d)$-matrix multiplication time.
We give a communication complexity lower bound that rules out all oblivious matrix sketching approaches.

In Section~\ref{sec:hardness-weights}, to circumvent the difficulty raised in Section~\ref{sec:hardness-matrix}, we consider a weaker problem where the algorithm only need to find a set of good weights (instead of a $d \times d$ matrix).
We give a reduction to show that this problem is still at least as hard as some basic matrix computation question, which we do not know how to solve faster than $(d, d^2, d)$-matrix multiplication time.

\subsection{Approximating the Empirical Covariance Matrix}
\label{sec:hardness-matrix}
Our algorithm matches the running time of the best non-robust covariance estimation algorithm.
When there are no corrupted samples and $N = \tilde \Omega(d^2 / \eps^2)$, with high probability, the empirical second-moment matrix $\frac{1}{N} \sum_{i=1}^N X_i X_i^\top$ is $\eps$-close to the true covariance matrix in Frobenius norm.
However, it is not known how to (approximately) compute this empirical second-moment matrix faster than $(d, d^2, d)$ matrix multiplication time.

\begin{problem}[Approximating Matrix Products]
\label{prob:empirical}
Let $d > 0$ and $N = \Omega(d^2)$.
Given $X \in \R^{N \times d}$ where each column of $X$ is drawn from $\NN(0, \Sigma)$ for some unknown $\Sigma \preceq I$,
  compute a matrix $\hat\Sigma$ such that
\[ \fnorm{\hat\Sigma - \frac{1}{N} X^\top X} = O(1). \]
\end{problem}

For approximate matrix product of an $N \times d$ matrix $A$ with $\normtwo{A} = O(1)$, we want to choose a sketching matrix $S$ so that $\fnorm{A^\top S^\top S A - A^\top A}^2 = O(1)$.
Known results for approximate matrix product state that if $S$ has $s$ rows, then $\fnorm{A^\top S^\top S A - A^\top A}^2 = O\left(\frac{\fnorm{A}^4}{s}\right)$ with probability at least $9/10$, see, e.g., Section 2.2 of~\cite{Woodruff14} for a survey.
In the context of Problem~\ref{prob:empirical}, letting $A = \frac{1}{\sqrt{N}} X$, we have $\normtwo{A} = O(1)$ and $\fnorm{A}^4 = O(d^2)$.
The error is then $O(d^2/s)$, and consequently $S$ must have $s = \Omega(d^2)$ rows for the error to be at most $O(1)$.

We can show that the argument above is almost tight for all oblivious sketches.
\begin{lemma}
\label{lem:noJL}
Let $N = d^2$.
There is no distribution over $t \times N$ matrices $S$, oblivious to the underlying input $N \times d$ matrix $A$, where $t = o(d^2/\log d)$, such that with probability at least $2/3$, it holds that $\|A^\top S^\top SA-A^\top A\|_F^2 \leq C_1 \frac{\|A\|_F^4}{d^2}$, where $C_1 = 4 \cdot 25 \cdot 2000^2$ is a positive constant.
\end{lemma}
\begin{proof} Suppose, to the contrary, there were such a distribution on matrices $S$ satisfying $t = o(d^2/\log d)$.

For $N = d^2$, consider a uniformly random $N \times d$ matrix $A \in \{-\frac{1}{d}, \frac{1}{d}\}^{N \times d}$. Then $\|A\|_F^2 = d$, and so for a random matrix $S$ from our family and a random input $A$ from this family of inputs, it holds that with probability at least $\frac{2}{3}$, $\|A^\top S^\top SA-A^\top A\|_F^2 \leq C_1 \frac{\|A\|_F^4}{d^2} = C_1$. By anti-concentration of the binomial distribution, with probability at least $\frac{99}{100}$, at least a $\frac{99}{100}$-fraction of the off-diagonal entries of $A^\top A$ have absolute value at least $\frac{1}{1000d}$.

Consequently, for at least a $\frac{24}{25}$-fraction of the entries in the bottom left $\frac{d}{2} \times \frac{d}{2}$ submatrix of $A^\top A$, we have the property that the entry has the same sign as in $A^\top S^\top SA$, and also the entries in $A^\top S^\top SA$ are at least $\frac{1}{2000d}$. Indeed, otherwise we would have $\|A^\top S^\top SA-A^\top A\|_F^2 > (\frac{d}{2})^2 \cdot \frac{1}{25} \cdot (\frac{1}{1000d} - \frac{1}{2000d})^2 > C_1$ with probability at least $\frac{99}{100}$ over the choice of $A$ and $S$, and in particular there exists a fixed $A$ for which this holds with probability at least $\frac{99}{100}$ over the choice of $S$, contradicting our assumption on the family of matrices $S$.

Now consider the following two-player communication game with public shared randomness. Alice has the first $\frac{d}{2}$ columns of $A$, denoted $A_L \in \{-1,1\}^{N \times d/2}$ while Bob has the remaining $\frac{d}{2}$ columns of $A$, denoted $A_R \in \{-1,1\}^{N \times d/2}$. The entries in the in the bottom left $\frac{d}{2} \times \frac{d}{2}$ submatrix of $A^\top A$ are exactly the inner products between all columns of Alice and all columns of Bob. Suppose there were such a family of matrices $S$ as described above. Alice and Bob use the public coin to agree upon $S$ with no communication. Alice then computes $S \cdot A_L$, and rounds each entry to the nearest power of $(1+\frac{1}{\poly(d)})$. Note all entries of $S \cdot A_L$ need to be at most $\poly(d)$ and rounding preserves $\|A^\top S^\top SA-A^\top A\|_F^2$ up to additive $\frac{1}{\poly(d)}$. Therefore, we maintain the property that, at least a $\frac{24}{25}$-fraction of the entries in the bottom left $\frac{d}{2} \times \frac{d}{2}$ submatrix of $A^\top A$ have the same signs in $A^\top A$ and $A^\top S^\top SA$. Alice sends each of the rounded entries of $S \cdot A_L$, which is $\Theta(t d \log d) = o(d^3)$ bits. Bob then computes $S \cdot A_R$ and thus forms $S \cdot A$, from which he can compute $A^\top S^\top SA$. At this point, Bob can recover the sign of a uniformly random entry in the bottom left $d/2 \times d/2$ submatrix of $A^\top A$ with probability at least $2/3-1/25-1/100 > 3/5$.

Notice that the sign of such an entry is the same as solving the Gap-Hamming communication problem under the uniform distribution: in this communication problem there are two players, Alice and Bob, who hold uniformly random vectors $x, y \in \{-1,1\}^N$, respectively, and wish to decide if $\langle x, y \rangle > 0$ or $\langle x, y \rangle < 0$. This problem requires $\Omega(N)$ randomized communication complexity \cite{ChakrabartiR12}. Moreover, as shown by Braverman et al. \cite{BravermanGPW16}, the information complexity of this problem is $\mathcal{I} = \Omega(N)$ bits. In our setting, we can think of Alice as having $\frac{d}{2}$ independent instances $x^1, \ldots, x^{d/2}$, and Bob having an index $i \in \{1, 2, \ldots, \frac{d}{2}\}$ as well as a vector $y$ and Bob wants to solve the Gap-Hamming problem on the pair $(x^i,y)$. However, only Alice is allowed to speak, and she sends a single message to Bob, without knowing $i$. By standard direct sum arguments in communication complexity \cite{BravermanR11} (see also \cite{PaghSW14} where Gap-Hamming composed with the Index problem was used), the randomized one-way communication complexity of this problem is $\Omega(d \cdot \mathcal{I}) = \Omega(d^3)$ bits. However, the communication cost of our protocol is $\Theta(t d \log d) = o(d^3)$ bits, which is a contradiction. Consequently, we must have $t = \Omega(d^2/\log d)$, as desired.
\end{proof}

It is worth noting that this lower bound   holds for any possible algorithm one can run on $S A$ (i.e., the algorithm can do more than just computing $A^\top S^\top S A$), so it is a stronger information-theoretic statement.

\subsection{Finding Good Weights}
\label{sec:hardness-weights}
To circumvent the difficulty of Problem~\ref{prob:empirical}, we could redefine our problem so that the algorithm does not need to output a $d \times d$ matrix,
  instead it outputs a set of good weights $w$ such that $\| \sum_{i=1}^N w_i X_i X_i^\top - \Sigma \|_F = O(\sqrt{\eps})$.
We will show that, even for this weaker problem of finding good weights, one still need to come up with faster algorithms for a basic matrix problem. 
\begin{problem}[Identifying Columns with Larger Norms in a Product Matrix]
\label{prob:apxnorm}
Let $d > 0$ and $N = \Omega(d^2)$.
Fix an arbitrary $U \in \R^{d \times \frac{d}{2}}$ with orthonormal columns.
Let $\Sigma'$ be defined as in Equation~\eqref{eqn:hardness-noise}.
Let $X \in \R^{d \times N}$ be a matrix with $(1-\eps) N$ columns drawn from $D = \NN(0, I)$ and $\eps N$ columns drawn from $D' = \NN(0, \Sigma')$.
Let $B$ be a set of $\eps N$ columns from $D'$.
Find a set $S \subset [N]$ such that $|S| \le 2\eps N$ and $|S \cap B| \ge \frac{|B|}{2}$.
\end{problem}

Consider the following instance.
Let $U$ be an arbitrary $d \times \frac{d}{2}$ matrix with orthonormal columns.
Let the good distribution be $D = \NN(0, \Sigma = I)$, and the noise distribution $D'$ is defined as
\begin{align}
\label{eqn:hardness-noise}
\textstyle D' &= \NN(0, \Sigma'), \quad \text{ where } \Sigma' = \frac{1}{1 + \frac{c}{\eps\sqrt{d}}} \left(I + \frac{2c}{\eps \sqrt{d}} U U^\top\right) \text{ for some } c = O(\log^{1/2} d) \; .
\end{align}
We draw $(1-\eps)N$ samples from $D$ and $\eps N$ samples from $D'$.
The empirical covariance matrix of the mixed distribution is
$\hat\Sigma = (1-\eps) \Sigma + \eps\Sigma' = \left(1-\frac{c}{\sqrt{d} + 1/\eps}\right) I + \left(\frac{2c}{\sqrt{d} + 1/\eps}\right) U U^\top$.
Observe that $\fnorm{\hat\Sigma - \Sigma}^2 = d \left(\frac{c}{\sqrt{d} + 1/\eps}\right)^2 = \Omega(1)$, so the bad samples are distorting the empirical covariance matrix by more than we could tolerate.

Observe that the good and bad samples have similar $\ell_2$-norm:
$\expect{X' \sim D'}{\normtwo{X'}} = \tr(\Sigma') = d = \expect{X \sim D}{\normtwo{X}}$.
However, the bad samples have slightly larger norm in the column space of $U$: \begin{align*}
\textstyle \expect{X \sim D}{\normtwo{U^\top X}^2} &= \tr(U U^\top) = \frac{d}{2} \; , \text{ and} \\
\textstyle \expect{X' \sim D'}{\normtwo{U^\top X'}^2} &= \tr(U U^\top) = \frac{d}{2} \cdot \frac{1 + \frac{2c}{\eps \sqrt{d}}}{1 + \frac{c}{\eps \sqrt{d}}} \ge \frac{d}{2} \left(1 + \frac{0.9 c}{\eps\sqrt{d}}\right) \; .
\end{align*}

Therefore, a natural way of distinguishing them is to compute $U^\top X$, which requires $(d, d^2, d)$-matrix multiplication time.  We could compute the column norms of $S U^\top A$, where $S$ is a Johnson-Lindenstrauss matrix.
However, $S$ must have $\frac{1}{\eps^2}$ rows to obtain $(1+\eps)$-approximation, and therefore $S$ must have $\tilde \Omega(d^2)$ rows.
Even if one uses a sparse matrix $S$, one has that $S U^\top$ is a dense matrix, and it is unclear how to compute $S U^\top A$ quickly.

Finally, we show that for this specific instance, any algorithm that can find a set of good weights $w \in \R^N$ must solve Problem~\ref{prob:apxnorm}.
\begin{lemma}
\label{lem:weights}
Consider the same setting as in Problem~\ref{prob:apxnorm}.
Given a set of weights $w$ such that $\normone{w} = 1$, $\norminf{w} \le \frac{1}{(1-\eps)N}$, and $\| \sum_{i=1}^N w_i X_i X_i^\top - I \|_F = O(1)$,
  we can solve Problem~\ref{prob:apxnorm} in $O(N)$ time.
\end{lemma}
\begin{proof}
Let $G$ and $B$ denote the set of good and bad samples respectively.
We have shown that $\expect{i \in G}{\normtwo{U^\top X_i}^2} = \frac{d}{2}$, and $\expect{i \in B}{\normtwo{U^\top X_i'}^2} \ge \frac{d}{2} \left(1 + \frac{0.9 c}{\eps\sqrt{d}}\right)$.
By standard concentration result of Chi-squared distributions, we know that there exists $c = O(\log^{1/2} d)$ such that with high probability,
\[ \forall i \in G, \quad \normtwo{U^\top X_i}^2 \ge \frac{d}{2}\left(1 - \frac{0.1 c}{\sqrt{d}}\right) \; , \text{ and} \]
\[ \forall i \in B, \quad \normtwo{U^\top X_i}^2 \ge \frac{d}{2}\left(1 + \frac{0.8 c}{\eps \sqrt{d}}\right) \; . \]
For the rest of proof we assume the samples meet these conditions.

Let $\Sigma_w = \sum_{i=1}^N w_i X_i X_i^\top$.
Let $w_G$ and $w_B$ denote the total weights on $G$ and $B$ respectively.
Since $\fnorm{\Sigma_w - I} = O(1)$, by Cauchy-Schwarz,
\[ U^\top (\Sigma_w - I) U = (U U^\top) \bullet (\Sigma_w - I) \le \fnorm{U U^\top} \fnorm{\Sigma_w - I} = O(\sqrt{d}) \le 0.05c \cdot \sqrt{d} \; . \]
On the other hand,
\begin{align*}
U^\top (\Sigma_w - I) U
&= \sum_{i=1}^N w_i \normtwo{U^\top X_i}^2 - \frac{d}{2} \\
&\ge w_B \left(\frac{d}{2} \cdot \frac{0.8c}{\eps\sqrt{d}}\right) - w_G \left(\frac{d}{2} \cdot \frac{0.1 c}{\sqrt{d}}\right) \\
&\ge \frac{w_B}{\eps} (0.4 c \cdot \sqrt{d}) - (0.05 c \cdot \sqrt{d}) \; .
\end{align*}
Putting these two inequalities together, we get that $w_B \le \frac{\eps}{4}$.
In other words, the average weight of a bad sample is $\frac{1}{4 N}$.

Let $S = \{i \in [N]: w_i \le \frac{1}{2N} \}$.
By Markov's inequality, we have $|S \cap B| \ge \frac{|B|}{2}$.
Since $\norminf{w} \le \frac{1}{(1-\eps)N}$ and $w_G = \normone{w} - w_B \ge 1 - \frac{\eps}{4}$, again by Markov's inequality, we get that $|S \cap G| \le \eps N$ and hence $|S| \le |B| + \eps N = 2 \eps N$.
\end{proof}

\section*{Acknowledgments}
This work was done in part while some of the authors were visiting the Simons Institute for the Theory of Computing.
Ilias Diakonikolas was supported by NSF Award CCF-1652862 (CAREER) and a Sloan Research Fellowship.
Rong Ge is supported by NSF Award CCF-1704656, CCF-1845171 (CAREER), a Sloan Research Fellowship, and a Google Faculty Research Award.
David Woodruff was supported in part by Office of Naval Research (ONR) grant N00014-18-1-2562.

\bibliographystyle{alpha}
\bibliography{allrefs}

\appendix

\makeatletter{}
\section{Omitted Proofs from Section~\ref{sec:prelim}}
\label{app:prelim}

\medskip
{\noindent \bf Lemma~\ref{lem:tensor-cov}~~}
{\em
Let $X \sim \NN(0, \Sigma)$ and $Z = X \otimes X$.  Let $\Sigma_Z \in \R^{d^2 \times d^2}$ be the covariance matrix of $Z$.
We have
\begin{enumerate}
\item[(i)] If $\Sigma \preceq I$, then $\Sigma_Z \preceq 2 I$.
\item[(ii)] If $0 \le \tau < 1$ and $\normtwo{\Sigma - I} \le \tau$, then $\normtwo{\Sigma_Z - 2 I} \le 6 \tau$.
\end{enumerate}
}
\medskip
\begin{proof}
Let $a \in \R^{d^2}$ be any unit vector.
Note that
\[ \normtwo{\Sigma_Z} = \max_{a \in \R^{d^2}, \normtwo{a} = 1} a^\top \Sigma_Z a \; . \]
Let $A$ be the unique matrix such that $A^\flat = v$.
We have
\[
v^\top \Sigma_Z v = \var{v^\top Z} = \var{X^\top A X} = \tr\left(A \Sigma (A + A^\top) \Sigma\right) \; .
\]
Note that $\Sigma$ is a covariance matrix, so it is always symmetric and PSD.
We can write $\Sigma$ as $\Sigma = \sum_{i=1}^d \lambda_i v_i v_i^\top$.
Let $\lambda_{\max}$ and $\lambda_{\min}$ denote the maximum and minimum eigenvalues of $\Sigma$.
Let $\hat A = \frac{A + A^\top}{2}$.
The right hand side is equal to
\begin{align*}
a^\top \Sigma_Z a
 = \tr\left(A \Sigma (A + A^\top) \Sigma\right)
&= 2\tr\left(\hat A \Sigma \hat A \Sigma\right) \\
&= 2 \sum_{i, j} \lambda_i \lambda_j \tr(\hat A v_i v_i^\top \hat A v_j v_j^\top) \\
&\le 2 (\lambda_{\max})^2 \sum_{i, j} (v_i^\top \hat A v_j)^2 \\
&\le 2 (\lambda_{\max})^2 \; .
\end{align*}
The last step uses the fact that $\sum_{i, j} (v_i^\top \hat A v_j)^2 = \fnorm{V \hat A V^\top}^2 = \|\hat A\|_F^2 \le \fnorm{A}^2 = \normtwo{a}^2 = 1$.
Since this holds for all unit vector $a \in \R^{d^2}$, we have $\Sigma_Z \preceq 2 (\lambda_{\max})^2 I$.
Similarly, we can prove that $\Sigma_Z \succeq 2 (\lambda_{\min})^2 I$.

For (i), by assumption $\lambda_{\max} = \normtwo{\Sigma} \le 1$, so we have $\normtwo{\Sigma_Z} \le 2$.

For (ii), we know $1 - \tau \le \lambda_{\min} \le \lambda_{\max} \le 1 + \tau$ and $0 < \tau < 1$.
It follows that $(1 - 2\tau) 2 I \preceq \Sigma \preceq (1 + 3\tau) 2 I$, and thus $\normtwo{\Sigma_Z - 2I} \le 6\tau$.
\end{proof}

\section{Omitted Proofs from Section~\ref{sec:cov}}
\label{app:cov}
\medskip
{\noindent \bf Lemma~\ref{lem:start}~~}
{\em
Let $D \sim \mathcal{N}(\mathbf{0}, \Sigma)$ be a zero-mean unknown covariance Gaussian on $\R^d$.
Let $\kappa$ denote the condition number of $\Sigma$.
Let $0 < \eps < \eps_0$ for some universal constant $\eps_0$.
Given an $\eps$-corrupted set of $N = \tilde \Omega(d^2 / \eps^2)$ samples drawn from $D$, we can compute a matrix $\Sigma_0$ in $\tilde O(d^{3.26}/\eps^2)$ time such that, with high probability, $\Sigma \preceq \Sigma_0 \preceq (\kappa \poly(d)) \Sigma$ and $\normtwo{\Sigma_0} \le \poly(d) \normtwo{\Sigma}$.
}
\medskip
\begin{proof}
Let $(G_i)_{i=1}^N$ be the original set of good samples drawn from $\NN(0, \Sigma)$, and let $(X_i)_{i=1}^N$ be the corrupted samples.
Let $S$ denote the set of $(1-\eps)N$ samples with the smallest norm $\normtwo{X_i}$.
We define $\Sigma_0 = 2 \left(\frac{1}{N} \sum_{i \in S} X_i X_i^\top\right)$.

We first show that $\Sigma_0 \succeq \Sigma$ with high probability.
Since the adversary corrupts at most $\eps N$ samples and we throw away $\eps N$ samples,
  we are left with at least $(1-2\eps)N$ good samples in $S$.
We will use the fact that removing any $(2\eps)$-fraction of the good samples will not change the empirical covariance too much.
Let $Y_i = \Sigma^{-1/2} G_i$ so that if $G_i \sim \NN(0, \Sigma)$ then $Y_i \sim \NN(0, I)$.
When $N = \tilde \Omega(d / \eps^2)$, for any $T \subset [N]$ with $|T| = (1-2\eps)N$, we have that with high probability,
\[
\normtwo{\frac{1}{N} \sum_{i \in T} Y_i Y_i^\top - I} \le O(\eps \log(1/\eps)) \; .
\]
We set $T \subseteq S$ to be a set of $(1-2\eps)N$ good samples in $S$, i.e., $G_i = X_i$ for all $i \in T$.
Let $M = \frac{1}{N} \sum_{i \in T} Y_i Y_i^\top$.
We know that $M = \frac{1}{N} \sum_{i \in T} \Sigma^{-1/2} G_i G_i^\top \Sigma^{-1/2}$ by definition,
  and $M \succeq (1 - O(\eps \log(1/\eps))) I \succeq \frac{I}{2}$ by the above concentration inequality.
Therefore,
\[
\Sigma_0
 = \frac{2}{N} \sum_{i \in S} X_i X_i^\top
 \succeq \frac{2}{N} \sum_{i \in T} X_i X_i^\top
 = \frac{2}{N} \sum_{i \in T} G_i G_i^\top
 = 2 \left(\Sigma^{1/2} M \Sigma^{1/2}\right)
 \succeq \Sigma \; .
\]

Next we show that $\normtwo{\Sigma_0} \le \poly(d) \normtwo{\Sigma}$ and $\Sigma_0 \preceq (\kappa \poly(d)) \Sigma$.
Let $\sigma^2$ denote the largest eigenvalue of $\Sigma$.
Again let $Y_i = \Sigma^{-1/2} G_i$, we know that when $N = \tilde \Omega(d / \eps^2)$, with high probability,
\[
\forall i \in [N], \quad \normtwo{Y_i} \le O(\sqrt{d \log d}) \; .
\]
We assume this condition holds for the rest of the proof.
As a result, $\normtwo{G_i} = \normtwo{\Sigma^{1/2} Y_i} \le O(\sigma \sqrt{d \log d})$ for all $i$.
Since only corrupted samples can have larger norm, and we remove the $\eps N$ samples with the largest norm,
  all samples in $S$ have norm at most $O(\sigma \sqrt{d \log d})$.
This gives an upper bound on the spectral norm of $\Sigma_0$,
\[
\normtwo{\Sigma_0} \le \frac{2}{N} \sum_{i \in S} \normtwo{X_i X_i^\top} \le \frac{2}{N} \cdot N \cdot \max_{i \in S} \normtwo{X_i}^2 = O(\sigma^2 d^2 \log^2 d) \; .
\]
This proves $\normtwo{\Sigma_0} \le \poly(d) \normtwo{\Sigma}$.
Moreover, by the definition of condition number we know that $\Sigma \succeq \frac{\sigma^2}{\kappa} I$,
  which implies $\Sigma_0 \preceq (\sigma^2 \poly(d)) I \preceq (\kappa \poly(d)) \Sigma$.

We conclude the proof by noting that $\Sigma_0$ can be computed by multiplying a $d \times |S|$ matrix with an $|S| \times d$ matrix. This can be done in time $\tilde O(d^{3.26}/\eps^2)$ by fast rectangular matrix multiplication (Lemma~\ref{lem:fast-mult}).
\end{proof}

\makeatletter{}
\subsection{Robust Covariance Estimation: Additive Approximations}
\label{sec:main-additive}
In this section, we prove Theorem~\ref{thm:main-additive}.
At a high level, we first use Lemma~\ref{lem:first-loop} to get a $O(\sqrt{\eps})$ additive approximation (Algorithm~\ref{alg:rough}).
We will then run this algorithm on subspaces that have much smaller eigenvalues in order to improve the guarantee (Algorithm~\ref{alg:additive}).
More precisely, we partition $\R^d$ into three disjoint subspaces $S_1$, $S_2$, and $S_3$, then we use Corollary~\ref{cor:add-m1}, Lemma~\ref{lem:add-s12},~and~Lemma~\ref{lem:add-s13} to learn the covariance in each component separately, and combine them together to get the final answer.
The fact that we only need an additive approximation 
is crucial for this approach.

For a subspace $S$, we use $\Pi_S$ to denote the projection matrix that maps $x \in \R^d$ onto $S$,
  and $S^\perp$ to denote the orthogonal complement of $S$.
Given a matrix $A$ and two subspaces $S_1$ and $S_2$, we use $\Sigma[S_1, S_2] = \Pi_{S_1} A \Pi_{S_2}$ to denote the projection of the rows and columns of $A$ onto $S_1$ and $S_2$ respectively. We write $A[S]$ for $A[S, S]$.

\begin{algorithm}[h]
  \caption{Robust Covariance Estimation for Gaussian Distributions (Additive Error)}
  \label{alg:additive}
  \SetKwInOut{Input}{Input}
  \SetKwInOut{Output}{Output}
  \Input{$0 < \eps < \eps_0$, and an $\eps$-corrupted set of $N = \tilde \Omega(d^2/\eps^2)$ samples $(X_i)_{i=1}^N$ drawn from $\NN(0, \Sigma)$.}
  \Output{A matrix $\hat \Sigma \in \R^{d \times d}$ such that with high probability, $\fnorm{\hat \Sigma - \Sigma} \le O(\eps\log(1/\eps))\normtwo{\Sigma}$.}
  Compute $M_0$ by running Algorithm~\ref{alg:rough} on input $(X_i)_{i=1}^N$. \\
  Compute eigendecomposition of $M_0$, let $S_1$ be the subspace of all eigenvalues at least $C_1\sqrt{\eps}$. \\
  Compute $M_1$ by running Algorithm~\ref{alg:rough} on input $(\Pi_{S_1^\perp} X_i)_{i=1}^N$.\\
  Compute eigendecomposition of $M_1[S_1^\perp]$, let $S_2$ be the subspace of all eigenvalues at least $C_2\eps$. \\
  Let $S_3$ be $(S_1\oplus S_2)^\perp$. \\
  Compute $M_2$ by calling Algorithm~\ref{alg:main} on inputs $\{\Pi_{S_1\oplus S_2} X_i\}$. \\
  Compute $M_3$ by calling Algorithm~\ref{alg:rough} on inputs $\{(\sqrt{\eps}\Pi_{S_1} + \eps^{1/4}\Pi_{S_2} + \Pi_{(S_1\oplus S_2)^\perp})X_i\}$. \\
  Let $\hat{\Sigma} = M_1+M_2 - M_2[S_2] + \frac{1}{\sqrt{\eps}}\left(M_3[S_1,S_3] + M_3[S_3,S_1]\right)$.\\
  \Return {$\hat{\Sigma}$.}
\end{algorithm}

\begin{algorithm}[h]
  \caption{Crude Robust Covariance Estimation}
  \label{alg:rough}
  \SetKwInOut{Input}{Input}
  \SetKwInOut{Output}{Output}
  \Input{$0 < \eps < \eps_0$, and an $\eps$-corrupted set of $N = \tilde \Omega(d^2/\eps^2)$ samples drawn from $\NN(0, \Sigma)$.}
  \Output{A matrix $\hat \Sigma \in \R^{d \times d}$ such that with high probability, $\fnorm{\hat \Sigma - \Sigma} \le O(\sqrt{\eps})\normtwo{\Sigma}$.}
  Compute an initial upper bound $\Sigma_0$ with $\Sigma \preceq \Sigma_0$ and $\normtwo{\Sigma_0} \le \poly(d) \normtwo{\Sigma}$ using Lemma~\ref{lem:start}. \\
  Let $T = O(\log d)$. \\
  \For{$t = 0$ {\bf to} $T - 1$}{
    Compute $\Sigma_{t+1} \in \R^{d \times d}$ with $\Sigma \preceq \Sigma_{t+1} \preceq \Sigma + O(\sqrt{\eps}) \Sigma_t$ using Lemma~\ref{lem:first-loop}  on $\Sigma_t$.
  }
  \Return {$\hat{\Sigma} = \Sigma_T$.}
\end{algorithm}

Let us first prove the guarantee for the crude $O(\sqrt{\eps})$ additive estimation.

\newcommand{\Cr}{C_0}

\begin{lemma} \label{lem:rough}
Under the same setting as Theorem~\ref{thm:main-additive}, there exists universal constant $\Cr$ such that Algorithm~\ref{alg:rough} outputs an estimate $\hat{\Sigma}$ that satisfies $\fnorm{\hat\Sigma - \Sigma} = O(\sqrt{\eps}) \normtwo{\Sigma}$ with high probability in time $\tilde O(d^{3.26})/\poly(\eps)$. 
\end{lemma}
\begin{proof}
By Lemma~\ref{lem:start}, in Algorithm~\ref{alg:rough}, we can compute $\Sigma_0 \succeq \Sigma$ such that $\normtwo{\Sigma_0} \le \poly(d) \normtwo{\Sigma}$.
Lemma~\ref{lem:first-loop} allows us to iteratively compute $\Sigma_{t+1}$ such that $\Sigma \preceq \Sigma_{t+1} \preceq \Sigma + O(\sqrt{\eps}) \Sigma_t$.
It follows that $\normtwo{\Sigma_{t+1}} \le \normtwo{\Sigma} + O(\sqrt{\eps}) \normtwo{\Sigma_t}$, and thus after $O(\log d)$ iterations we have a matrix $\Sigma_T$ with $\normtwo{\Sigma_T} \le 2 \normtwo{\Sigma}$.
Using Lemma~\ref{lem:first-loop} with $\Sigma_T$, we can get a matrix $\hat\Sigma$ with $\|\hat\Sigma - \Sigma\|_F = O(\sqrt{\eps}) \normtwo{\Sigma_T} = O(\sqrt{\eps})\Sigma$.
The running time follows from the running time of Lemmas~\ref{lem:start}~and~\ref{lem:first-loop}.
\end{proof}

Suppose the constant hiding in the $O(\cdot)$ notation in Lemma~\ref{lem:rough} is $\Cr$, we have the following immediate corollary of Lemma~\ref{lem:rough}.
\begin{corollary}
\label{cor:add-m1}
In Algorithm~\ref{alg:additive}, the matrix $M_0$ satisfies $\fnorm{M_0 - \Sigma} \le \Cr\sqrt{\eps} \normtwo{\Sigma}$.
\end{corollary}

In Algorithm~\ref{alg:additive} we will choose $C_1 = 20 C_r$ 
and define $S_1$ to be the subspace where the eigenvalues of $M_0$ are at least $C_1\sqrt{\eps}$. We can then show the following lemma:

\begin{lemma}
\label{lem:add-m2}
In Algorithm~\ref{alg:additive}, with high probability, the matrix $M_1$ satisfies
\[ \fnorm{M_1 - \Sigma[S_1^\perp]} \le (2C_1+\Cr)\Cr \eps \normtwo{\Sigma} \; .\]
\end{lemma}
\begin{proof}
We assume the calls to compute $M_0, M_1$ are successful, which happens with high probability.

By Lemma~\ref{lem:rough}, we know 
\[ \fnorm{M_1 - \Sigma[S_1^\perp]} \le \Cr \sqrt{\eps} \normtwo{\Sigma[S_1^\perp]} \; .\]
We continue to bound $\normtwo{\Sigma[S_1^\perp]}$.
Notice that by Corollary~\ref{cor:add-m1},
\begin{align*}
\normtwo{\Sigma[S_1^\perp]} & \le \normtwo{M_0[S_1^\perp]} + \Cr\sqrt{\eps}\normtwo{\Sigma}  \\
& \le C_1\sqrt{\eps}\normtwo{M_0}+\Cr\sqrt{\eps}\normtwo{\Sigma}\\
& \le (2C_1+\Cr)\sqrt{\eps}\normtwo{\Sigma} \; .
\end{align*}

Here the last step uses the fact when $\eps_0$ is small enough $\normtwo{M_0} \le 2\norm{\Sigma}$. 
\end{proof}

Let $S_2$ denote the subspace of $S_1^\perp$ where the eigenvalues of $M_1[S_1^\perp]$ is at least $C_2\eps$ where $C_2 =20(2C_1+\Cr)\Cr$.
Let $S_3$ denote the orthogonal subspace of $S_1\oplus S_2$ (which corresponds to the eigenvectors of $M_1[S_1^\perp]$ that are smaller than $C_2\eps$). Note that $S_1$, $S_2$, and $S_3$ form a disjoint partition of $\R^d$.
We will learn $\Sigma$ separately on the product of these subspaces, and combine them together to get the final answer $\hat \Sigma$.
We use $S_{12}$ to denote the subspace $S_1 \oplus S_2$.

We will now show that $M_2$ computed by Algorithm~\ref{alg:additive} has low additive error in the subspace $S_{12}$.

\begin{lemma}
\label{lem:add-s12}
In Algorithm~\ref{alg:additive}, with high probability, the matrix $M_2$ satisfies
\[ \fnorm{M_2 - \Sigma[S_{12}]} \le O(\eps\log(1/\eps)\normtwo{\Sigma} \; .\]
Moreover, if we consider $\Pi_{S_{12}}X_i$ as vectors of dimension equal to the dimension of $S_{12}$, the algorithm runs in time $\tilde O(d^{3.26})/\poly(\eps)$.
\end{lemma}
\begin{proof}
We assume the calls for computing matrices $M_0,M_1, M_2$ are all successful, which happens with high probability.

Let $Y_i = \Pi_{S_{12}}X_i$, we know the covariance of these samples are exactly equal to $\Sigma[S_{12}]$. 

In this case, by the guarantee of Theorem~\ref{thm:main} we know
\[
\fnorm{M_2^{-1/2}\Sigma[S_{12}]M_2^{-1/2}-I} \le O(\eps\log(1/\eps)).
\]
We can left and right multiply by $M_2^{1/2}$ and get
\[
\fnorm{\Sigma[S_{12}] - M_2} \le O(\eps\log(1/\eps))\normtwo{M_2}.
\]
When $\eps$ is small enough this implies $\normtwo{M_2} \le 2 \normtwo{\Sigma[S_{12}]} \le 2 \normtwo{\Sigma}$, therefore as desired we have 
\[
\fnorm{\Sigma[S_{12}] - M_2} \le O(\eps\log(1/\eps))\normtwo{\Sigma}.
\]

The only thing left to establish is the running time. 
To bound the running time we will show $\kappa(\Sigma[S_{12}]) = O(1/\eps)$. Here we restrict the attention to the subspace $S_{12}$, so if $S_{12}$ as dimension $k$, $\kappa$ is the ratio of the largest eigenvalue and the $k$-th eigenvalue.

Let $S^\star$ be the subspace of eigenvectors of $\Sigma$ with eigenvalue at most $\eps\normtwo{\Sigma}$. We first show the following claim:

\newtheorem*{claim}{Claim}

\begin{claim}
For any unit vector $v \in S^\star$, $\normtwo{\Pi_{S_{12}}v}^2 \le 1/5$. 
\end{claim}

\begin{proof}
Since $\normtwo{\Pi_{S_{12}}v}^2 = \normtwo{\Pi_{S_1}v}^2+\normtwo{\Pi_{S_2}v}^2$, we will bound the contributions separately. Notice that for any $v\in S^\star$, we have $v^\top \Sigma v \le \eps$, therefore by Corollary~\ref{cor:add-m1},
\[
v^\top M_0 v \le v^\top \Sigma v + \normtwo{\Sigma - M_0} \le (\Cr\sqrt{\eps}+\eps)\normtwo{\Sigma}\le 2\Cr\sqrt{\eps}\normtwo{\Sigma}.
\] 
On the other hand, $v^\top M_0 v \ge C_1\sqrt{\eps}\normtwo{\Sigma} \normtwo{\Pi_{S_1}v}^2$. Combining the two equations we get  $\normtwo{\Pi_{S_1}v}^2\le 2C_r/C_1 = 1/10$. The proof for $\normtwo{\Pi_{S_2}v}^2$ is exactly the same except we use Lemma~\ref{lem:add-m2}.
\end{proof}

For any two subspaces $U$ and $V$, one can check that
\[
\sup_{u\in U,\|u\|_2=1} \normtwo{\Pi_Vu}^2 = \sup_{v\in V,\|v\|_2=1} \normtwo{\Pi_Uv}^2 = \cos^2 \inf_{u\in U,v\in V,\|u\|=\|v\|=1}\angle(u,v),
\]
where $\angle(u,v)$ is the angle between $u,v$. Therefore we know for any vector $v\in S_{12}$, $\normtwo{\Pi_{S^\star} v}^2 \le 1/5$. This implies for every $v\in S_{12}$,
\[
v^\top \Sigma v \ge \eps\normtwo{\Sigma} \normtwo{\Pi_{Z^\perp}v}^2 \ge \frac{4\eps}{5}.
\]

This shows $\lambda_k(\Sigma_{S_{12}}) \ge \frac{4\eps}{5} \normtwo{\Sigma}$, so $\kappa \le \frac{5}{4\eps} = O(1/\eps)$. \end{proof}

Finally we give the guarantee for $M_3$.

\begin{lemma}
\label{lem:add-s13}
In Algorithm~\ref{alg:additive}, with high probability, the matrix $M_3$ satisfies
\[ \fnorm{\frac{1}{\sqrt{\eps}}M_3[S_1,S_3] - \Sigma[S_1,S_3]} \le O(\eps)\normtwo{\Sigma} \; .\]
\end{lemma}
\begin{proof}
We assume the calls to compute $M_0, M_1, M_3$ are successful, which happens with high probability.

Set $Y_i = (\eps^{1/2} \Pi_{S_1} + \eps^{1/4} \Pi_{S_2} + \Pi_{S_3}) X_i$.
Let $\Sigma_Y$ denote the covariance of $Y$. It is easy to check that $\Sigma_Y \preceq 3\left(\eps\Sigma[S_1] + \sqrt{\eps}\Sigma[S_2] + \Sigma[S_3]\right)$. By Lemma~\ref{cor:add-m1}, 
\[
\normtwo{\Sigma[S_2]}\le \normtwo{\Sigma[S_1^\perp]} \le O(\sqrt{\eps})\normtwo{\Sigma}+\normtwo{M_0[S_1^\perp]} = O(\sqrt{\eps})\normtwo{\Sigma}.
\]
Similarly by Lemma~\ref{lem:add-m2},
\[
\normtwo{\Sigma[S_3]}\le O(\eps)\normtwo{\Sigma}+\normtwo{M_1[S_3]} = O(\eps)\normtwo{\Sigma}.
\]
Combining these we have $\normtwo{\Sigma_Y} \le O(\eps)\normtwo{\Sigma}$. Therefore by Lemma~\ref{lem:rough} we know the estimation $M_3$ satisfies $\fnorm{M_3 - \Sigma_Y} \le O(\eps^{1.5})\normtwo{\Sigma}$. 

On the other hand, it is easy to check that $\Sigma_Y[S_1,S_3] = \frac{1}{\sqrt{\eps}}\Sigma[S_1,S_3]$, therefore we have
\begin{align*}
\fnorm{\frac{1}{\sqrt{\eps}}M_3[S_1,S_3] - \Sigma[S_1,S_3]}
&=\frac{1}{\sqrt{\eps}}\fnorm{M_3[S_1,S_3] - \Sigma_Y[S_1,S_3]} \\
&\le \frac{1}{\sqrt{\eps}}\fnorm{M_3 - \Sigma_Y} = O(\eps)\normtwo{\Sigma}. \qedhere
\end{align*}
\end{proof}

Finally we are ready to combine all the steps.

\begin{proof}[Proof of Theorem~\ref{thm:main-additive}]
We will assume the four calls to Algorithm~\ref{alg:main} and Algorithm~\ref{alg:rough} are all successful, which happens with high probability. The running time follows from Lemma~\ref{lem:rough} and Lemma~\ref{lem:add-s12}. Now the resulting matrix looks like
\[
\left[
\begin{array}{ccc}
M_2[S_1] & M_2[S_1,S_2] & \frac{1}{\sqrt{\eps}}M_3[S_1,S_3] \\
M_2[S_2, S_1] & M_1[S_2] & M_1[S_2, S_3] \\
\frac{1}{\sqrt{\eps}}M_3[S_3,S_1] & M_1[S_3, S_2] & M_1[S_3]
\end{array}
\right]
\]
By Lemmas~\ref{lem:add-m2}, \ref{lem:add-s12}, \ref{lem:add-s13} we know for each one of these nine blocks the error is bounded by $O(\eps\log(1/\eps))\normtwo{\Sigma}$. Therefore, the entire matrix also has error at most $O(\eps\log(1/\eps))\normtwo{\Sigma}$.
\end{proof}

\subsection{Robust Mean Estimation for Bounded-Covariance Distributions}
\label{app:boundedcov}

We use the robust mean estimation algorithm for bounded-covariance distributions from~\cite{ChengDR19} to achieve Lemma~\ref{lem:mean-boundedcov}.

We state this algorithm (Algorithm~\ref{alg:boundedcov}) to be self-contained.

\begin{algorithm}[h]
  \caption{Robust Mean Estimation for Bounded Covariance Distributions}
  \label{alg:boundedcov}
  \SetKwInOut{Input}{Input}
  \SetKwInOut{Output}{Output}
  \Input{$0 < \eps < \eps_0$, and an $\eps$-corrupted set of $N = \tilde \Omega(d/\eps)$ samples $(Z_i)_{i=1}^N$ drawn from $D$. \\
    $D$ is the ground-truth distribution supported on $\R^d$ with mean $\mus$ and covariance $\Sigma \preceq I$.}
  \Output{A vector $\hat \mu \in \R^d$ such that, with high probability, $\normtwo{\hat \mu - \mus} \le O(\sqrt{\eps})$.}
  Let $\nu \in \R^d$ be an initial guess with $\normtwo{\nu - \mus} \le \poly(d)$. \\
  \For{$i = 1$ {\bf to} $O(\log d)$}{
   Compute a near-optimal solution $w \in \R^N$ to the primal SDP~\eqref{eqn:primal-sdp} with parameters $\nu$ and $2\eps$. \\
   Compute a near-optimal solution $M \in \R^{d \times d}$ for the dual SDP~\eqref{eqn:dual-sdp} with parameters $\nu$ and $\eps$.\\
   \eIf{the objective value of $w$ in SDP~\eqref{eqn:primal-sdp} is at most $c$ ($c$ is a universal constant)}{
     \Return{the weighted empirical mean $\hat \mu_w = \sum_{i=1}^N w_i Z_i$}.\\
   }{
     Move $\nu$ closer to $\mus$ using the top eigenvector of $M$. }
  }
\end{algorithm}

Notice that Algorithm~\ref{alg:boundedcov} is almost identical to Algorithm~\ref{alg:apxcov}, except the stopping criteria in the ``if'' statement.
Therefore, we can speed up Algorithm~\ref{alg:boundedcov} using Proposition~\ref{prop:runtime-tensor}, 
as we do for Algorithm~\ref{alg:apxcov}.

\subsection{Robust Mean Estimation with Approximately Known Covariance}
\label{app:mean}
In this section, we prove the error guarantee part of Lemma~\ref{lem:mean-apx-cov}, i.e., correctness of Algorithm~\ref{alg:apxcov}.
Note that we will not worry about running time here, so we can use the naive implementation of Algorithm~\ref{alg:apxcov} which runs in time $\tilde O(d^4)/\poly(\eps)$.
For the same reason, we ignore the additional structure in our input and focus on the mean estimation problem.
For the rest of this section, we use $d$ to denote the dimensionality of the problem, and $N = \tilde \Omega(d / \eps^2)$ to denote the number of samples.
(We have $d = (d')^2$ if we are trying to estimate the covariance matrix of a $(d')$-dimensional Gaussian.)

We use $(X_i)_{i=1}^N$ to denote the input, which is a set of $d$-dimensional $\eps$-corrupted samples drawn from some ground-truth distribution $D$.
We know $D$ has covariance matrix $\Sigma$ with $\normtwo{\Sigma - I} \le \tau$, and the goal is to estimate the unknown mean $\mus$ of $D$.
We first restate Lemma~\ref{lem:mean-apx-cov}.

\bigskip
{\bf \noindent Lemma~\ref{lem:mean-apx-cov}~~}
{\em
Let $D$ be a distribution supported on $\R^d$ with unknown mean $\mus$ and covariance $\Sigma$.
Let $0 < \eps < \eps_0$ for some universal constant $\eps_0$, $\tau \le O(\sqrt{\eps})$, and $\delta = O(\sqrt{\tau\eps} + \eps \log(1/\eps))$.
Suppose that $D$ has exponentially decaying tails, and $\Sigma$ is close to the identity matrix $\normtwo{\Sigma - I} \le \tau$.
Given an $\eps$-corrupted set of $N = \tilde \Omega(d / \delta^2)$ samples drawn from $D$,
  Algorithm~\ref{alg:apxcov} outputs a hypothesis vector $\hat \mu$ such that, with high probability, $\normtwo{\mu - \mus} \le O(\delta)$.
}
\medskip

We use $G^\star$ for the original set of $N$ good samples drawn from $D$.
After $\eps$-fraction of the samples are corrupted, we use $G \subseteq G^\star$ for the remaining good samples and $B$ for the corrupted samples.
The input to the algorithm is $G \cup B$.  We have $|G| \ge (1-\eps)N$ and $|B| \le \eps N$.
Let $\Delta_{N,\eps}$ denote the convex hull of all uniform distributions over subsets $S \subseteq [N]$ of size $|S| = (1-\eps)N$:
\[
\Delta_{N,\eps} = \left\{ w \in \R^N : \sum_{i=1}^N w_i = 1 \text{ and } 0 \le w_i \le \frac{1}{(1-\eps) N} \text{ for all } i \right\} \; .
\]
Every weight vector $w \in \Delta_{N,\eps}$ correspond to a fractional set of $(1-\eps)N$ samples.

By standard concentration results, we know that degree-$2$ polynomials of Gaussian random variables are exponentially concentrated around their mean.

\begin{definition}[Exponentially Decaying Tails]
\label{def:exp-decay}
We say a distribution $D$ supported on $\R^d$ has exponentially decaying tails iff, for any unit vector $v \in \R^d$, we have $\prob{Z \sim D}{\inner{v, Z - \mus} \ge t} \le \exp(-\Omega(t))$.
\end{definition}

To avoid dealing with the randomness of the good samples, we require the following deterministic conditions on the original set of $N$ good samples $G^\star$ (which hold with high probability when $N = \tilde \Omega(d/\eps^2)$ when $D$ satisfies Definition~\ref{def:exp-decay}).
For all $w \in \Delta_{N, 2\eps}$, we require the following conditions to hold for $\delta_1 = O(\eps \log(1/\eps))$ and $\delta_2 = O(\tau + \eps \log^2(1/\eps))$:
\begin{align}
\label{eqn:good-sample-moments}
\normtwo{\sum_{i \in G^\star} w_i (X_i - \mus)} \le \delta_1 \; , \; \quad
\normtwo{\sum_{i \in G^\star} w_i (X_i - \mus) (X_i - \mus)^\top - I} \le \delta_2 \; , \\
\label{eqn:good-pruning}
\forall i\in G^\star, \; \normtwo{X_i-\mus} \le O(\sqrt{d\log(d)}) \; .
\end{align}
At a high level, they state that with high probability, the good samples are never too far from $\mus$, and the empirical first and second moments of the good samples behave as we expect them to.
More specifically, $\delta_1$ upper bounds the change in the mean when we remove \emph{any} $\eps$-fraction of the samples,
  and $\delta_2$ upper bounds the change in the second-moment matrix.
The second-order condition follows from the fact that $\normtwo{\Sigma - I} \le \tau$, the triangle inequality for the spectral norm, and with high probability for our choice of $N$,
\[
\normtwo{\sum_{i \in G^\star} w_i (X_i - \mus) (X_i - \mus)^\top - \Sigma} \le O(\eps \log^2(1/\eps)) \normtwo{\Sigma} = O(\eps \log^2(1/\eps)) \; .
\]

We adapt the proof of~\cite{ChengDR19} to prove the following lemma, which holds for general distributions that satisfy the concentration bounds above.
\begin{lemma}
\label{lem:general-delta}
Assume the concentration bounds (Conditions~\eqref{eqn:good-sample-moments}~and~\eqref{eqn:good-pruning}) hold for the good samples with parameters $\delta_1$ and $\delta_2$ where $\delta_2 \ge \delta_1^2$.
Let $\delta = \sqrt{\eps \delta_2}$.
Then Algorithm~\ref{alg:apxcov}, with threshold $(1 + O(\delta_2))$ in the ``if'' statement, will output a weight vector $w$ such that the weighted empirical mean $\hat \mu_w = \sum_{i=1}^N w_i X_i$ satisfies $\normtwo{\mu - \hat \mu} \le O(\delta)$ for $\delta = O(\delta)$.
\end{lemma}

Lemma~\ref{lem:mean-apx-cov} follows immediately from Lemma~\ref{lem:general-delta}, because the output of Algorithm~\ref{alg:apxcov} has error $\delta = O(\sqrt{\eps \delta_2}) = O(\sqrt{\eps (\tau + \eps \log^2(1/\eps))}) = O(\sqrt{\eps \tau} + \eps \log(1/\eps))$ as needed.

Algorithm~\ref{alg:apxcov} is based on the primal-dual approach proposed by~\cite{ChengDR19} for robust mean estimation.
Their algorithm starts with a guess $\nu \in \R^{d}$, and then in each iteration solves the primal and dual SDPs~\eqref{eqn:primal-sdp}~and~\eqref{eqn:dual-sdp}.
They gave a win-win analysis: either a good primal solution gives weights $w \in \R^N$ such that $\hat \mu_w$ is close to the true mean;
  or a good dual solution must identify a direction of improvement that allows the algorithm to move $\nu$ much closer to the true mean.

\begin{lp}\tag*{(\ref{eqn:primal-sdp})}
\mini{\lambda_{\max} \left(\sum_{i=1}^N w_i (X_i - \nu) (X_i - \nu)^\top\right)} 
\st \con{w \in \Delta_{N, \eps}}
\end{lp}
\begin{lp}\tag*{(\ref{eqn:dual-sdp})}
\maxi{\text{average of the smallest $(1-\eps)$-fraction of $\left((X_i - \nu)^\top M (X_i - \nu)\right)_{i=1}^N$}}
\st \con{M \succeq 0, \tr(M) \le 1}
\end{lp}

To prove Lemma~\ref{lem:general-delta}, we will show that the win-win analysis still holds in our setting by proving two structural lemmas.
Lemma~\ref{lem:wrong-mean-primal-nosol} proves that a good primal solution for {\em any} guess $\nu$ will give an accurate weighted empirical mean.
Lemma~\ref{lem:good-dual-better-nu} shows that we can use the top eigenvector of a near-optimal dual solution to move $\nu$ closer to $\mus$ by a constant factor.

First we prove a helper lemma.
Lemma~\ref{lem:ub-lb-opt} gives upper and lower bounds on the optimal value of the SDPs \eqref{eqn:primal-sdp} and \eqref{eqn:dual-sdp}.
For example, Lemma~\ref{lem:ub-lb-opt} allows us to estimate how far $\nu$ is from $\mus$ from the optimal value of the SDPs.

\begin{lemma}[Optimal Value of the SDPs]
\label{lem:ub-lb-opt}
Fix $0 < \eps < \eps_0$, $\delta_1$, $\delta_2 \ge \delta_1^2$, and $\nu \in \R^d$.
Let $\delta = \sqrt{\eps \delta_2}$.
Let $\{X_i\}_{i=1}^N$ be an $\eps$-corrupted set of $N$ samples that satisfy Condition~\eqref{eqn:good-sample-moments}.
Let $\OPT_{\nu,\eps}$ denote the optimal value of the SDPs~\eqref{eqn:primal-sdp}~and~\eqref{eqn:dual-sdp} with parameters $\nu$ and $\eps$.
Let $r = \normtwo{\nu - \mus}$.
Then, we have:
\begin{align*}
(1-2\eps) \left((1 - \delta_2) + r^2 - 2 \delta_1 r \right) & \le \OPT_{\nu,\eps} \le (1 + \delta_2) + r^2 + 2\delta_1 r \; .
\end{align*}
In particular, when $\eps_0 < 1/20$ and $r = \Omega(\sqrt{\delta_2})$, we can simplify the above as
\[
1 + 0.9 r^2 \le \OPT_{\nu,\eps} \le 1 + 1.1 r^2 \; .
\]
\end{lemma}
\begin{proof}
Let $\OPT = \OPT_{\nu, \eps}$.

One feasible primal solution is to set $w_i = \frac{1}{|G|}$ for all $i \in G$ (and $w_i = 0$ for all $i \in B$):
\begin{align*}
\OPT & \le \lambda_{\max} \left( \sum_{i=1}^N w_i (X_i - \nu) (X_i - \nu)^\top \right)   = \max_{y \in \R^d, \normtwo{y}=1} \sum_{i\in G} w_i \inner{X_i - \nu, y}^2 \\
  & = \max_{y \in \R^d, \normtwo{y}=1} \left(\sum_{i\in G} w_i \inner{X_i - \mus, y}^2 + \inner{\mus - \nu, y}^2 + 2\inner{\sum_{i\in G} w_i (X_i - \mus), y}\inner{\mus - \nu, y}\right) \\
  & \le \max_{y \in \R^d, \normtwo{y}=1} \left((1+\delta_2) + \inner{\mus - \nu, y}^2 + 2\delta_1 \inner{\mus - \nu, y}\right) \\
  & = (1 + \delta_2) + \normtwo{\mus - \nu}^2 + 2\delta_1 \normtwo{\mus - \nu} \; .
\end{align*}
We used Condition~\eqref{eqn:good-sample-moments}, since $w$ can be viewed as a weight vector on $G^\star$ where $w \in \Delta_{N,\eps}$.

One feasible dual solution is $M = yy^\top$ where $y = \frac{\mus - \nu}{\normtwo{\mus - \nu}}$.
The dual objective value is the mean of the smallest $(1-\eps)$-fraction of $\left((X_i-\nu)^\top M (X_i-\nu)\right)_{i=1}^N$, which is at least
\[
\frac{1}{(1-\eps)N} \min_{S \subset G, |S| = (1-2\eps)N} \sum_{i\in S} (X_i-\nu)^\top M (X_i-\nu) \; .
\]
This is because the smallest $(1-\eps)N$ entries in $G$ must include the smallest $(1-2\eps)N$ entries.
Let $w'_i = \frac{1}{|S|}$ for all $i \in S$ and $w'_i = 0$ otherwise.
Note that $S \subset G$ and $|S| = (1-2\eps)N$, so $w'$ can be viewed as a weight vector on $G^\star$ with $w' \in \Delta_{N,2\eps}$.
Therefore we have \begin{align*}
\OPT & \ge \sum_{i\in S} \frac{1}{(1-\eps) N} (X_i-\nu)^\top M (X_i-\nu) = \frac{|S|}{(1-\eps)N} \sum_{i \in G} w'_i \inner{X_i - \nu, y}^2 \\
  & = \frac{1-2\eps}{1-\eps} \left( \sum_{i \in G} w'_i \inner{X_i - \mus, y}^2 + w'_G \normtwo{\mus - \nu}^2 + 2 \sum_{i \in G} w'_i \inner{X_i - \mus, y} \normtwo{\mus - \nu} \right) \\
  & \ge (1-2\eps) \left((1 - \delta_2) + \normtwo{\mus - \nu}^2 - 2 \delta_1 \normtwo{\mus - \nu} \right)\; . \end{align*}

To obtain the simpler upper and lower bounds, we note when $r = \Omega(\sqrt{\delta_2})$, the error term is $\delta_2 + 2\delta_1 r = O(r^2)$.
Therefore, by increasing the constant in $r = \Omega(\sqrt{\delta_2})$, we can get $1 + 0.9 r^2 \le \OPT \le 1 + 1.1 r^2$.
\end{proof}

Next we show that a good primal solution $w$ for any guess $\nu$ will give an accurate estimate $\hat \mu_w$.
Lemma~\ref{lem:wrong-mean-primal-nosol} proves the contrapositive statement: if the weighted empirical mean $\hat \mu_w$ is far from $\mus$, then no matter what our current guess $\nu$ is, $w$ cannot be a good solution to the primal SDP.

\begin{lemma}[Good Primal Solution $\Rightarrow$ Correct Mean]
\label{lem:wrong-mean-primal-nosol}
Fix $0 < \eps < \eps_0$, $\delta_1$, and $\delta_2 \ge \delta_1^2$.
Let $\delta = \sqrt{\eps \delta_2}$.
Let $\{X_i\}_{i=1}^N$ be an $\eps$-corrupted set of $N$ samples that satisfy Condition~\eqref{eqn:good-sample-moments}.
For all $w \in \Delta_{N,2\eps}$, if $\normtwo{\hat \mu_w - \mus} = \Omega(\delta)$ where $\hat \mu_w = \sum_{i=1}^N w_i X_i$, then for all $\nu \in \R^d$,
\[
\lambda_{\max} \left( \sum_{i=1}^N w_i (X_i - \nu) (X_i - \nu)^\top \right) \ge 1 + \Omega(\delta_2) \; . \]
\end{lemma}
\begin{proof}
Fix any $w \in \Delta_{N, 2\eps}$.
If $\normtwo{\mus - \nu} = \Omega(\sqrt{\delta_2})$, then because $w$ is feasible and by Lemma~\ref{lem:ub-lb-opt},
\[
\lambda_{\max} \left( \sum_{i=1}^N w_i (X_i - \nu) (X_i - \nu)^\top \right) \ge \OPT_{\nu,2\eps} \ge 1 + 0.9 \normtwo{\mus - \nu}^2 \ge 1 + \Omega(\delta_2) \; .
\]
Therefore, for the rest of this proof, we can assume $\normtwo{\mus - \nu} = O(\sqrt{\delta_2})$.

We project the samples along the direction of $(\hat \mu_w - \mus)$.
Consider the unit vector $y = (\hat \mu_w - \mus) / \normtwo{\hat \mu_w - \mus}$.
To bound from below the maximum eigenvalue, it is sufficient to show that
\[
y^\top \left( \sum_{i=1}^N w_i (X_i - \nu) (X_i - \nu)^\top \right) y = \sum_{i=1}^N w_i \inner{X_i - \nu, y}^2 \ge 1 + \Omega(\delta_2) \; .
\]
We first bound from below the contribution of the bad samples by $\Omega(\delta_2)$.
By triangle inequality, 
\begin{align*}
\abs{\sum_{i\in B} w_i \inner{X_i - \nu, y}}
  & \ge \abs{\sum_{i\in B} w_i \inner{X_i - \mus, y}} - w_B \abs{\inner{\mus - \nu, y}} \\
  & \ge \abs{\sum_{i=1}^N w_i \inner{X_i - \mus, y}} - \abs{\sum_{i\in G} w_i \inner{X_i - \mus, y}} - 2\eps \normtwo{\mus - \nu} \\
  & \ge \normtwo{\hat \mu_w - \mus} - \delta_1 - 2 \eps \cdot O(\sqrt{\delta_2}) = \Omega(\delta) \; .
\end{align*}
The last line follows from our choice of $y$, $\delta \ge \max(\delta_1, \eps\sqrt{\delta_2})$, and the good samples satisfy Condition~\eqref{eqn:good-sample-moments}.
By Cauchy-Schwarz,
\[
\left(\sum_{i\in B} w_i \inner{X_i - \nu, y}^2\right) \left(\sum_{i\in B} w_i\right) \ge \left(\sum_{i\in B} w_i \inner{X_i - \nu, y} \right)^2 = \Omega(\delta^2) \; .
\]
Since $w_B \le 2\eps$, we have $\sum_{i\in B} w_i \inner{X_i - \nu, y}^2 = \Omega(\delta^2/\eps) = \Omega(\delta_2)$.

We continue to lower bound the contribution of the good samples to the quadratic form by $1 - O(\delta_2)$.
By Condition~\eqref{eqn:good-sample-moments},
\begin{align*}
\sum_{i\in G} w_i \inner{X_i - \nu, y}^2
  & = \sum_{i\in G} w_i \left(\inner{X_i - \mus, y}^2 + \inner{\mus - \nu, y}^2 + 2\inner{X_i - \mus, y}\inner{\mus - \nu, y}\right) \\
  & \ge \sum_{i\in G} w_i \inner{X_i - \mus, y}^2 + 2 \inner{\mus - \nu, y} \inner{\sum_{i\in G} w_i (X_i - \mus), y} \\
  & \ge (1 - \delta_2) - 2\delta_1 \normtwo{\mus - \nu} = 1 - O(\delta_2) \; .
\end{align*}
In the last step, we used $\delta_1 \normtwo{\mus - \nu} \le 2 \delta_1 \sqrt{\delta_2} = O(\delta_2)$.
Putting together the contribution of good and bad samples, we have $\sum_{i=1}^N w_i \inner{X_i - \nu, y}^2 \ge 1 - O(\delta_2) + \Omega(\delta_2) = 1 + \Omega(\delta_2)$.
\end{proof}

Lemma~\ref{lem:wrong-mean-primal-nosol} guarantees that, any solution to the primal SDP whose objective value is at most $1 + O(\delta_2)$ will give good weights, and this is independent of our current guess $\nu$.

We now deal with the other possibility: the primal SDP has no good solution.
Lemma~\ref{lem:good-dual-better-nu} shows that in this case, we can solve the dual SDP~\eqref{eqn:dual-sdp} and move $\nu$ closer to $\mus$ by a constant factor.
We simplify the proof by assuming that we can solve the dual SDP exactly.
This assumption is wlog as shown in~\cite{ChengDR19}.

\begin{lemma}[Good Dual Solution $\Rightarrow$ Better $\nu$]
\label{lem:good-dual-better-nu}
Fix $0 < \eps < \eps_0$, $\delta_1$, and $\delta_2 \ge \delta_1^2$.
Let $\delta = \sqrt{\eps \delta_2}$.
Let $\{X_i\}_{i=1}^N$ be an $\eps$-corrupted set of $N$ samples that satisfy Condition~\eqref{eqn:good-sample-moments}.
If the optimal solution $M \in \R^{d \times d}$ to the dual SDP~\eqref{eqn:dual-sdp} (with parameters $\nu$ and $\eps$) has objective value at least $1 + \Omega(\delta_2)$, then we can efficiently find a vector $\nu' \in \R^d$, such that $\normtwo{\nu' - \mus} \le \frac{3}{4} \normtwo{\nu - \mus}$.
\end{lemma}
\begin{proof}
Because $M$ is a feasible solution to the dual SDP~\eqref{eqn:dual-sdp} with parameters $\nu$ and $\eps$, we know that $\OPT_{\nu,\eps} \ge 1 + \Omega(\delta_2)$.
When $\OPT_{\nu, \eps} \ge 1 + \Omega(\delta_2)$, Lemma~\ref{lem:ub-lb-opt} implies that $\normtwo{\mus - \nu} \ge \Omega(\sqrt{\delta_2})$ and $\OPT_{\nu, \eps} \ge 1 + 0.9 \normtwo{\mus - \nu}^2$.

We know $M \succeq 0$ and $\tr(M) = 1$.
Without loss of generality, we can assume $M$ is symmetric.

Since the objective value is the average of the smallest $(1-\eps)N$ entries of $(X_i - \nu)^\top M (X_i - \nu)$ and one way to choose $(1-\eps)N$ entries is to focus on the good samples, using Condition~\eqref{eqn:good-sample-moments},
\begin{align*}
1 + 0.9 & \normtwo{\mus - \nu}^2 \le \OPT_{\nu, \eps} \\
 & \le \frac{1}{|G|}\sum_{i\in G} (X_i - \nu)^\top M (X_i - \nu) \\
 & = \frac{1}{|G|}\sum_{i\in G} \inner{M, (X_i - \mus)(X_i - \mus)^\top + 2(X_i - \mus)(\mus-\nu)^\top + (\mus-\nu)(\mus-\nu)^\top} \\
 & \le 1 + \delta_2 + 2\delta_1 \normtwo{\mus-\nu} + \inner{M, (\mus-\nu)(\mus-\nu)^\top} \\
 & \le 1 + 0.1 \normtwo{\mus-\nu}^2 + \inner{M, (\mus-\nu)(\mus-\nu)^\top} \; .
\end{align*}
Therefore, we have $\inner{M, (\mus - \nu)(\mus - \nu)^\top} \ge \frac{4}{5} \normtwo{\mus-\nu}^2$.

We will continue to show that the top eigenvector of $M$ aligns with $(\nu - \mus)$.
Let $\lambda_1 \ge \lambda_2 \ge \ldots \ge \lambda_d \ge 0$ denote the eigenvalues of $M$, 
and let $v_1, \ldots, v_d$ denote the corresponding eigenvectors.
The conditions on $M$ implies that $\sum_{i=1}^d \lambda_d = 1$.
We decompose $(\mus - \nu)$ and write it as $\mus - \nu = \sum_{i=1}^d \alpha_i v_i$ where $\sum_{i=1}^d \alpha_i^2 = \normtwo{\mus - \nu}^2$.
Using these decompositions, we can rewrite $\inner{M, (\mus - \nu)(\mus - \nu)^\top} = \sum_{i=1}^d \lambda_i \alpha_i^2$.

First observe that $\lambda_1 \ge \frac{4}{5}$, because $\lambda_1 \sum_{i} \alpha_i^2 \ge \sum_{i} \lambda_i \alpha_i^2 \ge \frac{4}{5} \normtwo{\mus - \nu}^2 = \frac{4}{5}\sum_{i} \alpha_i^2$.
Moreover, because $\frac{4}{5} \sum_{i} \alpha_i^2 \le \sum_{i} \lambda_i \alpha_i^2 \le \lambda_1 \alpha_1^2 + (1-\lambda_1)(1-\alpha_1^2) \le \frac{4}{5} \alpha_1^2 + \frac{1}{5} \sum_i \alpha_i^2$, we know that $\inner{v_1 v_1^\top, (\mus - \nu)(\mus - \nu)^\top} = \alpha_1^2 \ge \frac{3}{4} \sum_{i} \alpha_i^2$.
Thus, we have a unit vector $v_1 \in \R^d$ with $\inner{v_1, \mus - \nu} = \alpha_1 \ge \frac{\sqrt{3}}{2} \normtwo{\mus - \nu}$, so the angle between $v_1$ and $\mus - \nu$ is at most $\theta \le \cos^{-1}(\frac{\sqrt{3}}{2})$.

Finally, we can estimate $r$ from the value of $\OPT_{\nu,\eps}$ using Lemma~\ref{lem:ub-lb-opt}, and move $\nu$ in a direction almost aligned with $\mus - \nu$, to obtain a new point $\nu'$ that is on a circle of radius $r' \approx r$ centered at $\nu$.
A basic geometric analysis shows that $\normtwo{\nu' - \mus} \le \frac{3}{4}\normtwo{\nu - \mus}$.
\end{proof}

\end{document}